\documentclass{article}

\usepackage{microtype}
\usepackage{graphicx}
\usepackage{epstopdf}
\usepackage{subfigure}
\usepackage{booktabs} % for professional tables

\usepackage{hyperref}

\usepackage[accepted]{icml2024}

\usepackage{amsmath}
\usepackage{amssymb}
\usepackage{mathtools}
\usepackage{amsthm}

\usepackage{tikz,siunitx}
\usetikzlibrary{positioning,calc,decorations.pathreplacing,decorations.markings,shapes.geometric,decorations.pathreplacing}
\usepackage{url}
\usepackage{subcaption}
\usepackage{svg}
\usepackage{multirow}

\usepackage[capitalize,noabbrev]{cleveref}

\theoremstyle{plain}
\newtheorem{theorem}{Theorem}[section]

\newtheorem{lemma}[theorem]{Lemma}
\newtheorem{corollary}[theorem]{Corollary}
\theoremstyle{definition}
\newtheorem{definition}[theorem]{Definition}

\theoremstyle{remark}

\usepackage[disable,textsize=tiny]{todonotes}

\icmltitlerunning{A Consistent Lebesgue Measure for Multi-label Learning}

\begin{document}

\twocolumn[
\icmltitle{A Consistent Lebesgue Measure for Multi-label Learning}



\icmlsetsymbol{equal}{*}

\begin{icmlauthorlist}
\icmlauthor{Kaan Demir}{}
\icmlauthor{Bach Nguyen}{}
\icmlauthor{Bing Xue}{}
\icmlauthor{Mengjie Zhang}{}
\end{icmlauthorlist}

%\icmlaffiliation{yyy}{Department of XXX, University of YYY, Location, Country}

\icmlcorrespondingauthor{Kaan Demir}{demirkaan@ecs.vuw.ac.nz}

\icmlkeywords{multi-label, consistency, surrogate-free optimisation}

\vskip 0.3in
]




\begin{abstract}
Multi-label loss functions are usually non-differentiable, requiring surrogate loss functions for gradient-based optimisation. The consistency of surrogate loss functions is not proven and is exacerbated by the conflicting nature of multi-label loss functions. To directly learn from multiple related, yet potentially conflicting multi-label loss functions, we propose a \textit{Consistent Lebesgue Measure-based Multi-label Learner} (CLML) and prove that CLML can achieve theoretical consistency under a Bayes risk framework. Empirical evidence supports our theory by demonstrating that: (1) CLML can consistently achieve state-of-the-art results; (2) the primary performance factor is the Lebesgue measure design, as CLML optimises a simpler feedforward model without additional label graph, perturbation-based conditioning, or semantic embeddings; and (3) an analysis of the results not only distinguishes CLML's effectiveness but also highlights inconsistencies between the surrogate and the desired loss functions.
\end{abstract}

\section{Introduction}
    In multi-label data, instances are associated with multiple target labels simultaneously. Multi-label learning is an important paradigm applicable to many real-world domains such as tabulated learning \citep{wu2017c2ae, bai2021mpvae, hang2022collaborative, lu2023distribution}, functional genomics \citep{patel2022modeling}, and computer vision \citep{Wang_He_Li_Long_Zhou_Ma_Wen_2020}. 
    Deep learning is responsible for many modern advancements in multi-label learning problems \citep{zhou2021deep, LIU2023103964}. 

    However, multi-label learning is usually considered challenging due to its complex label interactions. Techniques to superimpose label interactions on the weights of a deep learning model and condition networks against non-informative features via feature perturbation have been the primary concern of works such as \citet{Wang_He_Li_Long_Zhou_Ma_Wen_2020,hang2022dual, hang2022collaborative,wu2017c2ae, you2020cross, hang2022collaborative, yuan2023graph}, and \citet{yuan2023graph}. Generally speaking, deep learning for multi-label learning is dominated by computer vision methods, therefore drawing focus away from tabulated problems, an important area of multi-label learning \citep{wu2017c2ae, bai2021mpvae, patel2022modeling, hang2022collaborative, hang2022dual, lu2023distribution}.

    Multi-label learning is challenging due to the complexity of the output space. No existing loss function can quantify the quality of a label set universally. To exemplify this, consider the following loss functions: hamming loss, one minus the label ranking average precision, and one minus the micro-averaged $F_1$-score. All three loss functions ultimately pertain to multi-label accuracy \citep{han2023survey}. However, both the interpretation of quality and the learning behaviour can vary with the loss function selected \citep{wu2020multi,liu2021emerging}. The situation is worsened by the conflicting behaviour between loss functions \citep{wu2020multi}. Further, multi-label loss functions are themselves, typically, \textit{non-convex} and \textit{discontinuous}, which can be either challenging or impossible to optimise directly \citep{pmlr-v19-gao11a}. As a result, it is common to back-propagate on gradients obtained from a manually designed and differentiable surrogate loss function \citep{rumelhart1986learning, liu2021emerging}. However, both the learning behaviour and the solution itself are prescribed by the gradients of the chosen \textit{surrogate} loss function, which might not correspond to desired behaviour according to the desired loss \citep{10177983}. Finally, multi-label learning with surrogate loss functions is not always \textit{consistent} with what they are designed to approximate \citep{pmlr-v19-gao11a,liu2021emerging}.
    
    These challenges give rise to several interesting \textit{research questions}. First, how can a model learn directly from non-convex, discontinuous, or even non-differentiable loss functions without surrogates to avoid inconsistency? Second, how can a single model learn using \textit{multiple} related, yet potentially conflicting loss functions, thus achieving robust performance on a variety of tasks? Third, can such a method achieve theoretical consistency in the context of multi-label learning? Addressing the above questions is paramount to progressing the field of multi-label learning.

    These three important research questions motivate the design of a
    \textit{Consistent Lebesgue Measure-based Multi-label Learner} (\textit{CLML}), offering \textit{several advantages}. First, CLML learns from multiple related, yet potentially conflicting, loss functions using a single model. Second, CLML learns to solve the problem \textit{without} the use of a surrogate loss function. Third, our experimental findings demonstrate that CLML consistently achieves a $13.79\%$ to $58.33\%$ \textit{better} critical distance ranking against competitive state-of-the-art methods on a variety of loss functions and datasets. The empirical results are supported by our theoretical foundation that \textit{proves} the consistency of CLML when optimising several multi-label loss functions. The importance of CLML's approach is accentuated by its \textit{simple} representation, validating the importance of a consistent loss function for multi-label learning. Finally, our analysis of the optimisation behaviour suggests that CLML can consistently navigate the desired multi-dimensional loss landscape while naturally understanding and accounting for their trade-offs.

    The \textit{major contributions} of this work are as follows: (1) a novel approach to achieving multi-label learning with multiple loss functions; (2) a novel learning objective for several non-convex and discontinuous multi-label loss functions without the use of a surrogate loss function; (3) a proof showing that our method can theoretically achieve consistency; (4) results demonstrating that CLML with a simpler feedforward model representation can consistently achieve state-of-the-art results on a wide variety of loss functions and datasets; (5) analysis to solidify the importance of consistency in multi-label learning; and (6) an analysis that highlights how CLML can naturally consider the trade-offs between desired multi-label loss functions and the inconsistency between surrogate and desired loss functions.

    \section{Related works}
    \subsection{Multi-label learning}
    Multi-label learning is a common problem, including computer vision \citep{you2020cross,Wang_He_Li_Long_Zhou_Ma_Wen_2020, zhou2021deep, yuan2023graph, LIU2023103964}, functional genomics \citep{patel2022modeling}, and tabulated learning \citep{wu2017c2ae, bai2021mpvae, hang2022collaborative, hang2022dual, lu2023distribution}. Earlier work on multi-label learning transformed a multi-label problem into a series of single-label problems \citep{read2011classifier, liu2021emerging}. However, such a transformation does not typically perform well, as valuable label interactions are lost. Consequently, deep learning has advanced the multi-label learning field \citep{liu2017deep}. Researchers have followed this trend, exploiting feature interactions using self-attention mechanisms on transformers \citep{xiao2019label}, and deep latent space encoding of features and labels using auto-encoders (C2AE) \citep{wu2017c2ae}.
    
    Recently, dual perspective label-specific feature learning for multi-label classification (DELA) \citep{hang2022dual} and collaborative learning of label semantics and deep label-specific features (CLIF) \citep{hang2022collaborative} have significantly advanced the tabulated multi-label learning field. DELA trains classifiers that are robust against non-informative features using a stochastic feature perturbation framework, while CLIF exploits label interactions using label graph embedding. Both DELA and CLIF have achieved state-of-the-art results in multi-label learning, significantly outperforming many previously state-of-the-art works such as LIFT \citep{zhang2014lift}, LLSF \citep{huangllsf2015}, JFSC \citep{huang2017joint}, C2AE \citep{wu2017c2ae}, TIFS \citep{ma2020topic}, and MPVAE \citep{bai2021mpvae}.

    \subsection{Consistency} Multi-label classification metrics are non-differentiable and tend to vary (and in some cases conflict) in the interpretation of quality \citep{pmlr-v19-gao11a,liu2021emerging,wu2020multi}. This problem persists across all domains of multi-label learning, and therefore, all of the existing deep-learning methods that utilise gradient descent require a differentiable surrogate objective function, tending to approximate the desired classification performance without consistency guarantees \citep{pmlr-v19-gao11a}. Consistency in multi-label learning has been investigated in prior works \citet{pmlr-v19-gao11a} and \citet{wu2020multi}, both showing that partial approximations of a subset of loss functions are possible, raising questions regarding the feasibility of a consistent loss approximation, which remains an open-ended research question. A robust multi-label model should account for the inconsistencies between the multiple related, yet potentially conflicting, loss functions in multi-label learning by directly optimising them, as opposed to by approximation \citep{pmlr-v19-gao11a, wu2020multi, liu2021emerging}. In this paper, we are motivated to address this critical consistency issue using a Lebesgue measure-based approach to directly optimise a set of multi-label loss functions.

     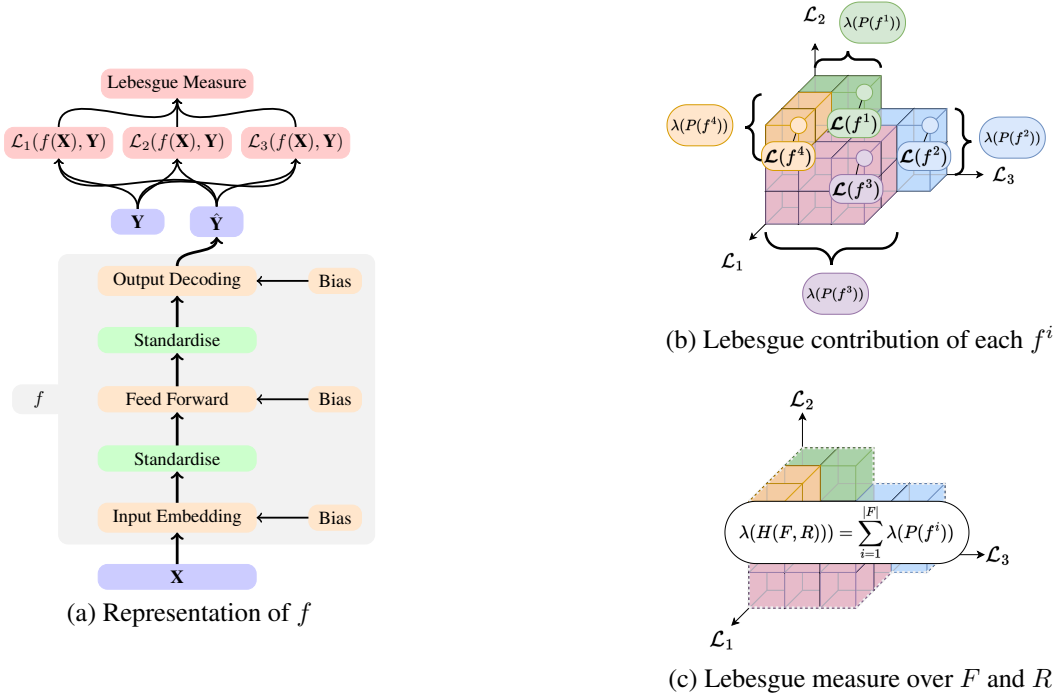
\begin{figure*}[!ht]
    \begin{minipage}{0.48\linewidth}
	\centering
    \scalebox{0.7}{
        \begin{tikzpicture}[scale=0.75,block/.style={rectangle,draw,minimum width=0.75cm,text centered},longblock/.style={rectangle,draw,minimum width=2.5cm, text centered}]

        \node[longblock,color=black!5,fill=black!5,text=black, minimum width=1cm,rounded corners=5pt] at (-3.5, 4.5) (modelf){$f$};

        \node[longblock,color=black!5,fill=black!5,text=black,minimum width=6cm,minimum height=5.5cm,rounded corners=5pt] at (1, 4.5) (background){};
        \node[longblock,color=blue!20,fill=blue!20,text=black, minimum width=3cm,rounded corners=5pt] at (0, 0) (inputdata){\textbf{X}};
        \node[longblock,color=orange!20,fill=orange!20,,text=black, minimum width=3cm,rounded corners=5pt] at (0,1.5) (encode0){Input Embedding};
        \node[longblock,color=orange!20,fill=orange!20,text=black, minimum width=1cm,rounded corners=5pt] at (4,1.5) (bias0){Bias};
        \node[longblock,color=green!20,fill=green!20,text=black,  minimum width=3cm,rounded corners=5pt] at (0,3) (norm1){Standardise};
        \node[longblock,color=orange!20,fill=orange!20,text=black, minimum width=3cm,rounded corners=5pt] at (0,4.5) (feedforward){Feed Forward};
        \node[longblock,color=orange!20,fill=orange!20,text=black, minimum width=1cm,rounded corners=5pt] at (4,4.5) (bias1){Bias};
        \node[longblock,color=green!20,fill=green!20,text=black,  minimum width=3cm,rounded corners=5pt] at (0,6) (norm2){Standardise};
        \node[longblock,color=orange!20,fill=orange!20,text=black,  minimum width=3cm,rounded corners=5pt] at (0,7.5) (out){Output Decoding};
        \node[longblock,color=orange!20,fill=orange!20,text=black, minimum width=1cm,rounded corners=5pt] at (4,7.5) (bias2){Bias};
        \node[longblock,color=blue!20,fill=blue!20,text=black,  minimum width=1cm,rounded corners=5pt] at (1,9) (pred){$\hat{\textbf{Y}}$};
        \node[longblock,color=blue!20,fill=blue!20,text=black,  minimum width=1cm,rounded corners=5pt] at (-1,9) (labels){$\textbf{Y}$};
        
        \draw[->,out=90,in=-90, line width=1.5pt] (inputdata.north) to (encode0.south);
        \draw[->,out=90,in=-90, line width=1.5pt] (encode0.north) to (norm1.south);
        \draw[->,out=90,in=-90, line width=1.5pt] (norm1.north) to (feedforward.south);
        \draw[->,out=90,in=-90, line width=1.5pt] (feedforward.north) to (norm2.south);
        \draw[->,out=90,in=-90, line width=1.5pt] (norm2.north) to (out.south);
        \draw[->,out=90,in=-90, line width=1.5pt] (out.north) to (pred.south);

        \draw[->,out=180,in=0, line width=1pt] (bias0.west) to (encode0.east);
        \draw[->,out=180,in=0, line width=1pt] (bias1.west) to (feedforward.east);
        \draw[->,out=180,in=0, line width=1pt] (bias2.west) to (out.east);

        \node[longblock,color=red!20,fill=red!20,text=black, minimum width=1cm,rounded corners=5pt] at (-3,11) (l1){$\mathcal{L}_1(f(\textbf{X}),\textbf{Y})$};
        \node[longblock,color=red!20,fill=red!20,text=black, minimum width=1cm,rounded corners=5pt] at (0,11) (l2){$\mathcal{L}_2(f(\textbf{X}),\textbf{Y})$};
        \node[longblock,color=red!20,fill=red!20,text=black, minimum width=1cm,rounded corners=5pt] at (3,11) (l3){$\mathcal{L}_3(f(\textbf{X}),\textbf{Y})$};

        \node[longblock,color=red!20,fill=red!20,text=black, minimum width=1cm,rounded corners=5pt] at (0,12.5) (lebesgue){Lebesgue Measure};

        \draw[->,out=90,in=-90, line width=1pt] (labels.north) to (l1.south);
        \draw[->,out=90,in=-90, line width=1pt] (labels.north) to (l2.south);
        \draw[->,out=90,in=-90, line width=1pt] (labels.north) to (l3.south);

        \draw[->,out=90,in=-90, line width=1pt] (pred.north) to (l1.south);
        \draw[->,out=90,in=-90, line width=1pt] (pred.north) to (l2.south);
        \draw[->,out=90,in=-90, line width=1pt] (pred.north) to (l3.south);

        \draw[->,out=90,in=-90, line width=1pt] (l1.north) to (lebesgue.south);
        \draw[->,out=90,in=-90, line width=1pt] (l2.north) to (lebesgue.south);
        \draw[->,out=90,in=-90, line width=1pt] (l3.north) to (lebesgue.south);

        \end{tikzpicture}
    }
    \centerline{(a) Representation of $f$}
    \end{minipage}
    \hfill
    \begin{minipage}{0.48\textwidth}
        \begin{minipage}{1\textwidth}
            \centering
            \includegraphics[width=0.7\textwidth]{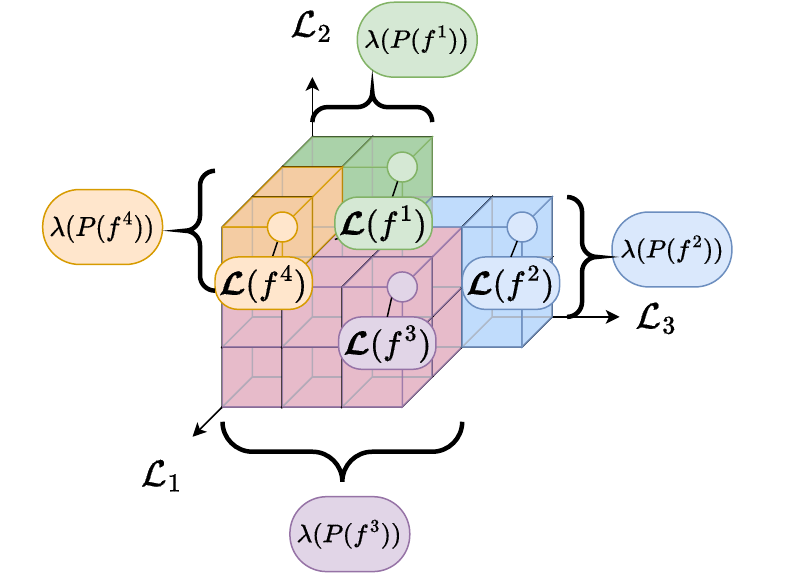}
            \centerline{(b) Lebesgue contribution of each $f^i$}
        \end{minipage}
        \vspace{10pt}
        \vfill
        \begin{minipage}{1\textwidth}
            \centering
            \includegraphics[width=0.7\textwidth]{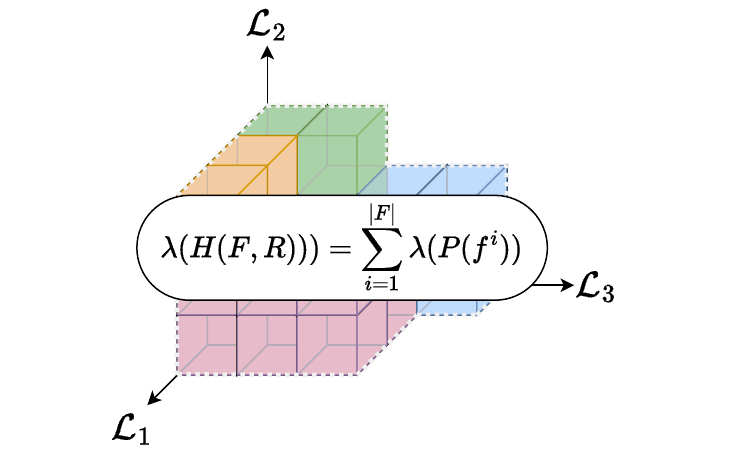}
        \centerline{(c) Lebesgue measure over $F$ and $R$}
        \end{minipage}
        \hfill
    \end{minipage}
    
    \caption{The overall proposed approach of CLML is outlined as follows. (a) illustrates the representation of $f$. (b) illustrates the contribution of each $f^i$ toward the improvement over all three loss functions $\boldsymbol{\mathcal{L}}(f^i) = (\mathcal{L}_1(f^i),\mathcal{L}_2(f^i),\mathcal{L}_3(f^i))$, which is quantified as the non-overlapping volume of space that $\boldsymbol{\mathcal{L}}(f^i)$ uniquely covers over a set of $W$ models $F=\cup_{i=1}^{W}\{f^i\}$, and a reference vector $R=\{1\}^3$. (c) illustrates the overall Lebesgue measure over $F$, which is the aggregate volume of all $f^i\in F$.}
    \label{attention}
    \vspace{-5pt}
    \end{figure*}

    \subsection{Multi-objective optimisation}

    Multi-objective optimisation has been \textit{overwhelmingly} applied to address certain areas of multi-label learning such as feature selection \cite{yin2015multi,dong2020many}. Thus, applying multi-objective optimisation to direct the learning of a classifier has seldom been addressed. In the single-label learning scenario, applying multi-objective classification methods such as \cite{zhao2018multi} to contemporary multi-label problems is impractical due to the genetic algorithm-based search, which is more suited to combinatorial optimisation than turning parameters for a deep learning model. Furthermore, multi-objective optimisation has been framed as multi-task optimisation in \cite{NEURIPS2018_432aca3a}, however, the problem is formulated to optimise a set of differentiable objective functions, reintroducing challenges concerning the consistency issues between the surrogate and desired objective functions, thereby posing a potential setback.
    
    A multi-objective fuzzy genetics-based multi-label learner was proposed in \cite{omozaki2022evolutionary} to address a balance in optimising accuracy-complexity tradeoffs in explainable classifier design. Fuzzy systems are known to be computationally expensive and ineffective at scale, greatly limiting their applicability to contemporary problems \cite{fazzolari2012review}. Nevertheless, multi-objective multi-label classification has been addressed in other works such as \cite{shi2014multi} to train radial basis functions using a genetic algorithm-based multi-objective framework. Notably, these efforts demonstrated promising performance when compared to traditional algorithms of the time. However, the quadratic computational complexity (scaling with the number of samples) limits the method's applicability to large-scale multi-label learning problems, as computation quickly becomes infeasible as the number of samples increases. 
    
    Existing research demonstrates the critical shortcomings of multi-label learning research concerning the direct optimisation of multiple conflicting loss functions, motivating the need for a deep-learning-based scalable, efficient, and effective approach that can address multi-label learning with \textit{proven} consistency.

    \section{The CLML approach}
    
    \subsection{Notations for multi-label classification}
    Multi-label learning is a supervised classification task, where an instance can be associated with multiple class labels simultaneously. Let $\mathcal{X}\in\mathbb{R}^D$, $\mathcal{Y}\in\{0,1\}^K$, and $\Omega\in\mathbb{R}^L$ respectively denote the input, output, and learnable parameter space for $D$ features, $K$ labels, and $L$ parameters. Let $\mathcal{P}$ be a joint probability distribution of samples over $\mathcal{X}\times\mathcal{Y}$. Let $f:\mathbb{R}^{D}\to \mathbb{R}^{K}$ represent a deep neural network drawn from $\Omega\in \mathbb{R}^{L}$, and trained on $N$ samples drawn from $\mathcal{P}$. An input vector $\textbf{x}\in\mathcal{X}$, where $\mathcal{X}\in\mathbb{R}^D$, can be associated with an output vector that is a subset of $\mathcal{Y}\in\{0,1\}^K$, \textit{i.e.}, $\textbf{y}=\{y_1,...,y_K\}$, where $y_l=1$ if label $l$ is associated with $\textbf{x}$, and is otherwise zero. We define the input feature and label data as $\textbf{X}\in\mathcal{X}^N$ and $\textbf{Y}\in\mathcal{Y}^N$, respectively. Given $\textbf{x}\in\mathcal{X}$, we denote $p(\textbf{y}|\textbf{x})_{\textbf{y}\in\mathcal{Y}}$ as the conditional probability of \textbf{y}. We additionally define $\kappa$ as the set of all conditional probabilities:
    \begin{equation}
        \kappa = \{p(\textbf{y}|\textbf{x}):\sum_{\textbf{y}\in\mathcal{Y}}p(\textbf{y}|\textbf{x})=1 \wedge p(\textbf{y}|\textbf{x})\geq 0\}.
    \end{equation}
    and the conditional risk of $f$ given surrogate loss ($\psi$), loss ($\mathcal{L}$), the conditional probability of sample $\textbf{x}$ and the label set $\textbf{y}$:
    \begin{equation}
    \begin{gathered}
         \mathcal{L}^{c}(p(\textbf{y}|\textbf{x}), f)=\sum_{\textbf{y}\in \mathcal{Y}}p(\textbf{y}|\textbf{x})\mathcal{L}(f(\textbf{x}),\textbf{y})\\
          \psi^c(p(\textbf{y}|\textbf{x}), f)=\sum_{\textbf{y}\in \mathcal{Y}}p(\textbf{y}|\textbf{x})\psi(f(\textbf{x}),\textbf{y}).
    \end{gathered}
    \end{equation}
    
    \subsection{The representation of CLML}

    Throughout this paper, we use a standard feedforward model to represent $f$, which is illustrated in Figure \ref{attention}. However, in comparison to a standard feedforward neural network, our model takes matrices as inputs and outputs, rather than individual vectors. This is due to the tabular nature of the data, allowing us to handle all samples simultaneously. First, the encoding layer $\textbf{E}:\mathbb{R}^{N\times D}\to \mathbb{R}^{N\times C}$ with bias $\textbf{W}_b^{\textbf{E}}$, compresses the input signal to $C$ embedding dimensions, where $C<<D$. An ablation study of $C$ is given in Section \ref{ablation}. Note that positional encoding is not required due to the tabulated nature of the data \citep{vaswani2017attention}. The compressed input is then row-standardised ($\gamma$) before being passed through a feedforward layer (\textbf{L}) with weights $\textbf{W}^{\textbf{L}}$ and bias $\textbf{W}_b^\textbf{L}$. We repeat standardisation before passing each row to the decoder $\textbf{D}:\mathbb{R}^{N\times C}\to \mathbb{R}^{N\times K}$ with bias $\textbf{W}_b^{\textbf{D}}$. The full equation for generating the prediction matrix $\hat{\textbf{Y}}$ is given by:
    \begin{equation}
        \hat{\textbf{Y}} = \sigma(\sigma(\gamma(\sigma(\gamma(\textbf{X}\textbf{E}+\textbf{W}_b^{\textbf{\textbf{E}}}))\textbf{W}^{\textbf{L}}+\textbf{W}_b^{\textbf{L}}))\textbf{D}+\textbf{W}_b^{\textbf{D}}).
    \end{equation}
    We apply a sigmoid activation function ($\sigma$) after each layer (and in particular the output layer, wherein a softmax is \textit{not} appropriate for multi-label data). The sigmoid function ensures bounded activations, which suits CLML's shallow and matrix-based representation. Furthermore, activation functions such as ReLU and GELU are more tailored to deeper architectures and address specific issues such as vanishing gradients, which is not as relevant in this paper. Tight-bound complexity scales \textit{linearly} with the parameters in the encoding $\Theta(NDC)$ and decoding stages $\Theta(NKC)$, and quadratically with the parameters of the feedforward step $\Theta(NC^2)$. Note the complexity assumes a naive implementation of matrix multiplication. We deliberately choose a simpler and shallow model to demonstrate the effectiveness of the consistent Lebesgue measure, described in the following subsection.
    
    \subsection{A Lebesgue measure for surrogate-free multi-label learning}
    The Lebesgue measure is widely used for multi-criteria, multi-task, and multi-objective problems \citep{igel2007covariance,bader2010faster,bader2011hype}. In comparison to traditional multi-objective methods such as dominance-based sorting, the Lebesgue measure can be expressed analytically, which is both useful for proving its consistency and applicable for future research to investigate its differentiation. We assume the learning task maps a batch (\textit{i.e.}, matrix) of input vectors $\textbf{X}^{N\times D}$ toward a batch of target labels $\textbf{Y}^{N\times K}$. Let $\boldsymbol{\mathcal{L}}(f(\textbf{X}),\textbf{Y}):\mathbb{R}^{N\times D}\to \mathbb{R}^o$ be a series of loss functions that map $\textbf{X}$ to a vector of losses $\boldsymbol{\mathcal{L}}(f(\textbf{X}),\textbf{Y}) = (\mathcal{L}_1(f(\textbf{X}),\textbf{Y}),\cdots,\mathcal{L}_o(f(\textbf{X}),\textbf{Y}))$ given the objective space $Z\subseteq \mathbb{R}^o$ and a representation of a neural-network $f$. Given $o=3$, let $\mathcal{L}_1, \mathcal{L}_2,$ and $\mathcal{L}_3$ respectively, without loss of generality, represent the following widely-used multi-label loss functions: hamming loss, one minus label-ranking average precision, and one minus micro $F_1$, all of which should be minimised (see \citet{liu2021emerging}). Let $R\subset Z$ denote a set of mutually non-dominating\footnote{$f$ dominates $f'$ ($f\prec f'$) iff $\forall i:1 \leq i \leq o: \mathcal{L}_i(f(\textbf{X}),\textbf{Y}) \leq \mathcal{L}_i(f'(\textbf{X}),\textbf{Y})$, and $\exists \mathcal{L}_i:\mathcal{L}_i(f(\textbf{X}),\textbf{Y})<\mathcal{L}_i(f'(\textbf{X}),\textbf{Y})$} loss vectors; $F$, the set of representations of functions; and $H(F,R)\subseteq Z$, the set of loss vectors that dominate at least one element of $R$ and are dominated by at least one element of $F$:
    \begin{equation}
        \begin{gathered}
        H(F,R) := 
        \{\textbf{z} \in Z \quad | \\\quad \exists f\in F ,\quad \exists \textbf{r} \in R: \boldsymbol{\mathcal{L}}(f(\textbf{X}),\textbf{Y}) \prec \textbf{z} \prec \textbf{r}\}.\label{lebesguemeasure}
        \end{gathered}
    \end{equation}
    
    The Lebesgue measure is defined as $\lambda(H(F,R)) = 
    \int_{\mathbb{R}^o}\textbf{1}_{H(F,R)}(\textbf{z})d\textbf{z}$, where $\textbf{1}_{H(F,R)}$ is the indicator function of $H(F,R)$.  The contribution of $f$ toward the improvement (minimisation) of a set of loss functions $\boldsymbol{\mathcal{L}}$ can be quantified by first measuring the improvement of $f$ via the partition function $P(f)$:
    \begin{equation}
        P(f) = H(\{f\},R)\backslash H(F\backslash\{f\},R).
    \end{equation}
    
    Hence, the Lebesgue contribution of $f$, $\lambda(P(f))=\int_{\mathbb{R}^o}\textbf{1}_{P(f)}(\textbf{z})d\textbf{z}$, describes it's contribution to minimising $\boldsymbol{\mathcal{L}}$. Ultimately, $\lambda(H(F,R))$, and therefore $\boldsymbol{\mathcal{L}}$, is sought to be optimised via $\lambda(P(f))$ (see Lemma \ref{lebesgue}), \textit{i.e.}, $f$ is guided by evaluating its Lebesgue contribution $\lambda(P(f))$, which can be efficiently estimated using Monte Carlo sampling \citep{bader2010faster}. $R$ is initialised to the unit loss vector $\{1\}^3$. Optimisation of CLML using the Lebesgue measure is achieved via covariance matrix adaptation \citep{hansen1996adapting}, a standard non-convex optimiser \citep{smith2020revisiting, sarafian2020explicit, nomura2021warm}. Please refer to Section \ref{optimisationprocess} for an expanded technical exposition on covariance matrix adaptation (see Section \ref{covariancematrixadaptation}), Lebesgue measure estimation using Monte Carlo sampling (see Section \ref{lebesguemeasureestimation}) and the optimisation process (see Section \ref{optimisationprocedure} and Algorithm \ref{optimisationalgorithm}). Although this work focuses on tabulated learning, CLML, and its framework, can be applied to other contemporary multi-label problems in future work, which falls outside of the scope of this paper.

    \section{Consistency of the Lebesgue measure}\label{Theorem2Proof}
    There are many definitions of consistency including infinite-sample consistency in \citet{zhang2004statisticala} and edge consistency in \citet{duchi2010edgeconsistency}. In this work, we define consistency as the Bayes risk, following \citet{pmlr-v19-gao11a}. We provide the following key definitions before introducing multi-label consistency.

    \begin{definition}[Conditional Risk]
    The expected conditional risk $R$, and the Bayesian risk $R^{B}$, of a model representation $f$ given $\mathcal{L}$ is defined as:
    \begin{equation}
    \begin{gathered}
        R(f) = \mathbb{E}_{(\textbf{x},\textbf{y})\sim \mathcal{P}}[\mathcal{L}^c(p(\textbf{y}|\textbf{x}), f)]\\ 
        R^{B}(f) = \mathbb{E}_{(\textbf{x},\textbf{y})\sim \mathcal{P}}[\underset{f'}{\text{inf}}[\mathcal{L}^c(p(\textbf{y}|\textbf{x}), f')]].
    \end{gathered}
    \end{equation}
    \end{definition}
    \begin{definition}[Bayes Predictors] The set of Bayes predictors:
    
    \begin{equation}
        B(p(\textbf{y}|\textbf{x}))=\{f:\mathcal{L}^c(p(\textbf{y}|\textbf{x}),f)=\underset{f'}{\text{inf}}[\mathcal{L}^c(p(\textbf{y}|\textbf{x}),f')]\}.
    \end{equation} determine that $\psi$ can be multi-label consistent w.r.t. $\mathcal{L}$ if the following holds for every $p(\textbf{y}|\textbf{x})\in \kappa$:
    \begin{equation}
        R^B_\psi (f) <\underset{f}{\text{inf}}\{R_\psi(f):\forall f \in \Omega, f\notin B\}.
    \end{equation}
    \end{definition}
    \begin{theorem}[Multi-label Consistency]\label{theoremold} $\psi$ can only be multi-label consistent w.r.t. $\mathcal{L}$ iff it holds for any sequence of $f^{(n)}$ that:
    \begin{equation}
        R_\psi (f^{(n)}) \to R^B_\psi (f) \quad \text{then} \quad R_\mathcal{L}(f^{(n)}) \to R_\mathcal{L}^B (f).
    \end{equation}   
    \end{theorem}
    The proof of Theorem \ref{theoremold} is available in \citet{pmlr-v19-gao11a}.
    \begin{definition}[Pareto optimal set] A Pareto optimal set of functions $\mathbb{P}^B$ contain the following functions:
    \begin{equation}
    \mathbb{P}^B = \{f: \{f': f'\prec f \quad \forall f',f \in \Omega, f' \neq f\}=\emptyset\}.
    \end{equation}   
    \end{definition}
     Recall that $f$ is said to dominate $f'$ ($f\prec f'$) iff $\forall i: 1 \leq i \leq o: \mathcal{L}_i(f(\textbf{X}),\textbf{Y}) \leq \mathcal{L}_i(f'(\textbf{X}),\textbf{Y})$, and $\exists \mathcal{L}_i:\mathcal{L}_i(f(\textbf{X}),\textbf{Y})<\mathcal{L}_i(f'(\textbf{X}),\textbf{Y})$.
    \begin{theorem}[A Consistent Lebesgue Measure]\label{theoremconsistency} Given a sequence $F^{(n)}$,
    the maximisation of the Lebesgue measure $\lambda(H(F^{(n)},R))$ is consistent with the minimisation of $\mathcal{L}_1, \mathcal{L}_2,$ and $\mathcal{L}_3$:
    \end{theorem}
    \begin{equation}
    \begin{gathered}
        \underset{n\to\infty}{\text{lim}}\lambda(H(F^{(n)},R))\to \lambda(H(\mathbb{P}^{B},R))\quad 
        \text{then} \\
        R_{\mathcal{L}_1}(f^{(n)}) \to R^B_{\mathcal{L}_1}(f) \wedge \\
         R_{\mathcal{L}_2}(f^{'(n)}) \to R^B_{\mathcal{L}_2}(f') \wedge \\
         R_{\mathcal{L}_3}(f^{''(n)}) \to R^B_{\mathcal{L}_3}(f'').
    \end{gathered}
    \end{equation}
    In other words, the maximisation of $\lambda(H(F^{(n)},R))$ tends to the convergence toward the Bayes risk for each loss function $\mathcal{L}_i$ $\forall i:1\leq i \leq 3$, $f^{(n)},f^{'(n)},f^{''(n)}\in F^{(n)}$ and that $f,f',f''\in\mathbb{P}^B$. 
    
    \begin{proof}[Proof of Theorem \ref{theoremconsistency}]
    
    We proceed by contradiction. Suppose the following function exists: $f^\gamma\notin \mathbb{P}^B$, $f^\gamma \in \Omega$ s.t. $\exists v: R_{\mathcal{L}_v}(f^\gamma) = R_{\mathcal{L}_v}^B(f^\gamma)$, \textit{i.e.}, $f^\gamma$ is a Bayes predictor for the $v^{th}$ loss $\mathcal{L}_v$ given $p(\textbf{y}|\textbf{x})$. Now suppose another function $f^\beta\in \Omega$ exists s.t. $f^\beta\in \mathbb{P}^B$. By this condition, $f^\beta \prec f^\gamma$ as $f^\gamma\notin \mathbb{P}^B$, hence $\forall i:1 \leq i \leq o: \mathcal{L}_i(f^\beta) \leq \mathcal{L}_i(f^\gamma)$, and $\exists \mathcal{L}_k:\mathcal{L}_k(f^\beta)<\mathcal{L}_k(f^\gamma)$. This result has two implications:

    \begin{enumerate}
        \item If $k=v$ then $\mathcal{L}_k(f^\beta)<\mathcal{L}_k(f^\gamma)$ would contradict $f^\gamma$ being a Bayes predictor. This would imply a Bayes predictor \textit{cannot exist outside} $\mathbb{P}^B$.
        \item If $k\neq v$, then $\forall i: 1 \leq i \leq o: \mathcal{L}_i(f^\beta) \leq \mathcal{L}_i(f^\gamma)$. For this condition to hold, when $i=v$, $\mathcal{L}_i(f^\beta) \leq \mathcal{L}_i(f^\gamma)$ would imply that $f^\beta$ is \textit{also} a Bayes predictor of $\mathcal{L}_i$, when there is strict equality, and implication 1 when there is inequality. Therefore, the Bayes predictor of $\mathcal{L}_i$ already exists within $\mathbb{P}^B$.
    \end{enumerate}

    The Pareto optimal set of representations therefore contains a set of Bayes predictors, one for each loss dimension. Hence, given a sequence $F^{(n)}$ (abstracting away the classifier),
    the maximisation of the Lebesgue measure $\lambda(H(F^{(n)},R))$ is eventually consistent with the minimisation of $\mathcal{L}_1, \mathcal{L}_2,$ and $\mathcal{L}_3$.
    \end{proof}

    Additional and useful definitions, corollaries, and lemmas are elaborated in Section \ref{usefuldefinitionscorollariesandlemmas} of the appendix. The following section provides significant empirical evidence to support the consistency of the Lebesgue measure.

    \section{Evaluation of the multi-label classification performance} 

    \subsection{Comparative studies}

    We compare CLML against several state-of-the-art and benchmark methods using the recommended parameter configurations in their papers. Our first comparison method is, to the best of our knowledge, the current state-of-the-art and best tabulated multi-label learner to date: dual perspective of label-specific feature learning for multi-label classification (DELA) \citep{hang2022dual}. DELA constructs classifiers that are robust against non-informative features by training on randomly perturbed label-specific non-informative features. Our second comparison method is collaborative learning of label semantics and deep label-specific features (CLIF) \citep{hang2022collaborative}. CLIF embeds label interactions using label graph isomorphism networks. Generally speaking, DELA and CLIF have achieved state-of-the-art results in comparison to many well-known tabulated multi-label learners \citep{zhang2014lift, huangllsf2015, huang2017joint, wu2017c2ae, ma2020topic, bai2021mpvae}.
    
     Our third comparison method learns deep latent spaces for multi-label classification (C2AE) \citep{wu2017c2ae}, which jointly encodes features and labels into a shared semantic space via an autoencoder network. Our fourth comparison method is ML$k$NN, which is a multi-label variant of $k$NN that estimates labels based on Bayesian inference \citep{zhang2007ml}. The final two comparison methods are based on Gaussian Naive Bayes with binary relevance (GNB-BR) and classifier chain (GNB-CC) transformations \citep{read2011classifier}.     

     \begin{figure}[!ht]
        \centering
        \begin{minipage}{0.32\textwidth}
            \includegraphics[width=0.9\textwidth]{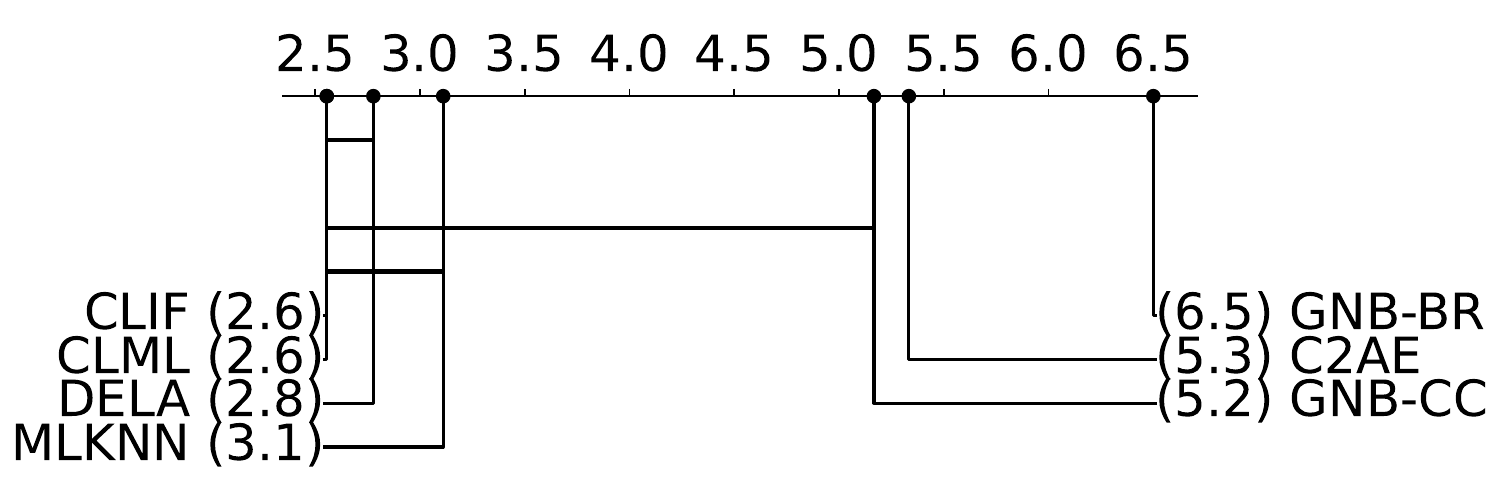}
            \centerline{(a) $\mathcal{L}_1(f)$}
        \end{minipage}
        \hfill
        \begin{minipage}{0.32\textwidth}
            \includegraphics[width=0.9\textwidth]{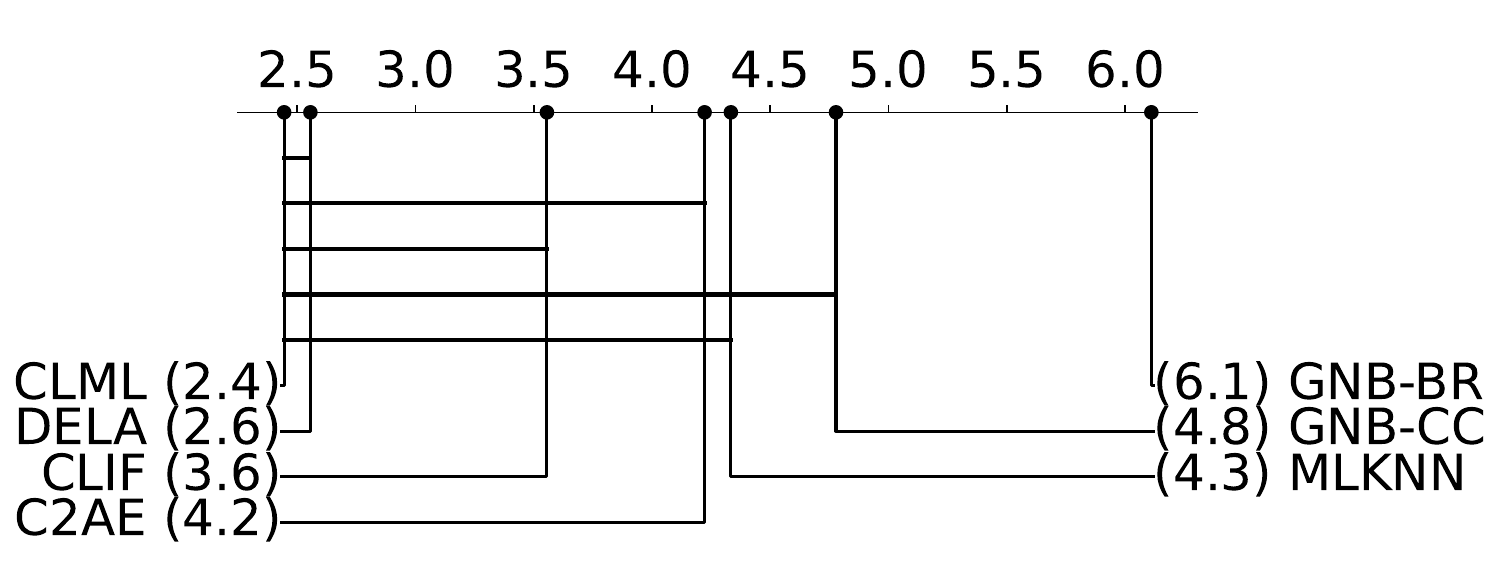}
            \centerline{(b) $\mathcal{L}_2(f)$}
        \end{minipage}
        \hfill
        \begin{minipage}{0.32\textwidth}
            \includegraphics[width=0.9\textwidth]{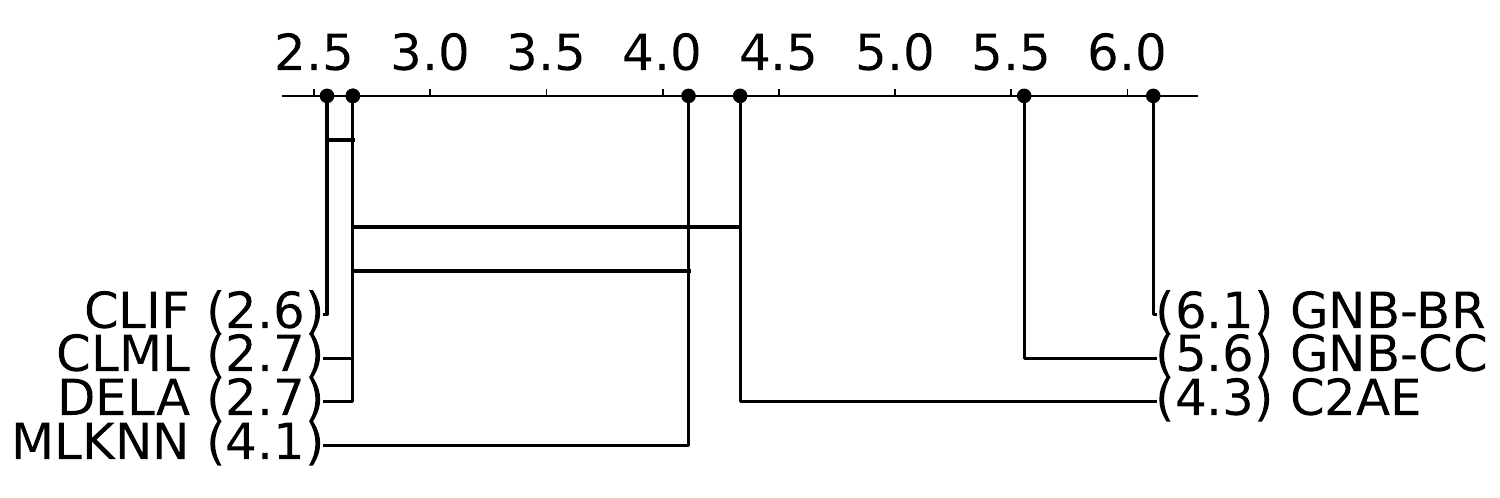}
            \centerline{(c) $\mathcal{L}_3(f)$}
        \end{minipage}
        \begin{minipage}{0.32\textwidth}
            \includegraphics[width=0.9\textwidth]{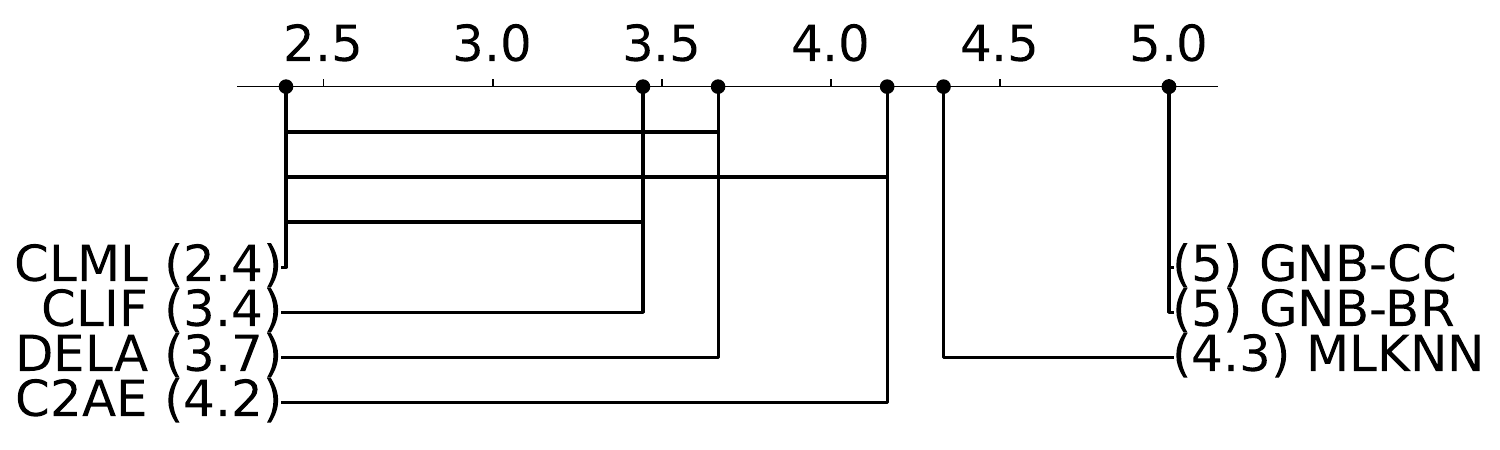}
            \centerline{(d) $\lambda(P(f))$}
        \end{minipage}
        \hfill
        \begin{minipage}{0.32\textwidth}
            \includegraphics[width=0.9\textwidth]{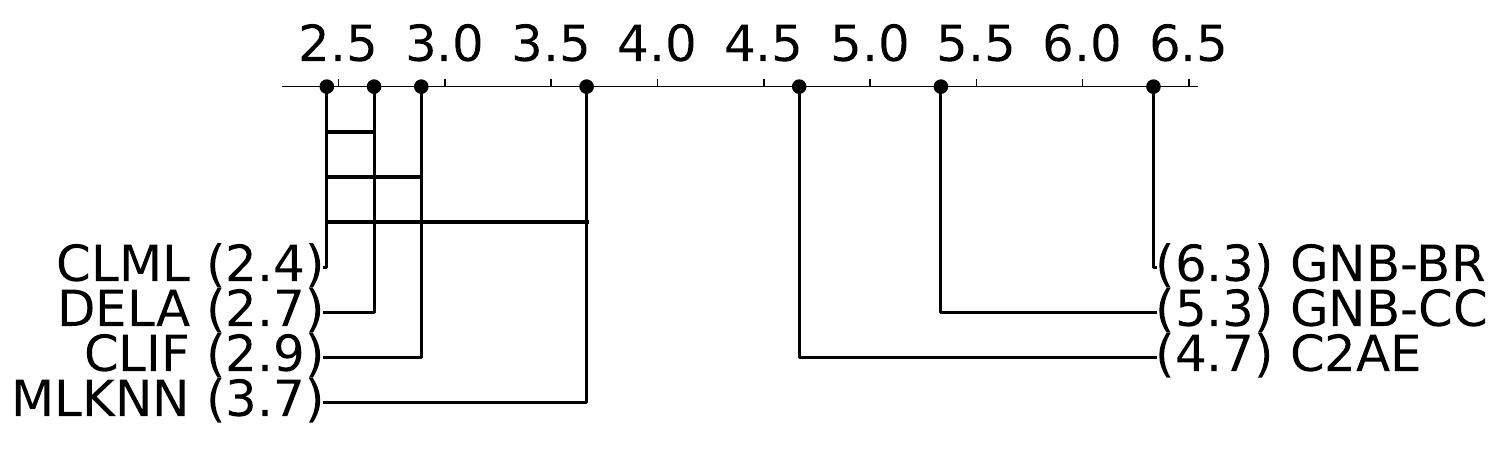}
            \centerline{(e) $\mu_g(\boldsymbol{\mathcal{L}}(f))$}
        \end{minipage}
        \hfill
        \caption{Bonferroni-Dunn test critical difference plots. A crossbar is drawn between CLML and any method if their difference in average ranking is less than the critical difference ($CD=2.686$ with $K=7$ methods and $T=9$ datasets obtained from a studentised range table).}
        \label{CDBonferroni}
    \end{figure}
     
    \subsection{Statistical analysis}
    
    We compare CLML against DELA, CLIF, C2AE, ML$k$NN, GNB-CC, and GNB-BR using several multi-label loss functions on nine different, commonly used open-access datasets. The experimental protocol and dataset summary are described in detail in Section \ref{experiments}. 
    
    \subsubsection{Aggregated performances}
    The geometric mean of the loss vector $\boldsymbol{\mathcal{L}}(f(\textbf{X}),\textbf{Y})$,
    $\mu_g(\boldsymbol{\mathcal{L}}(f(\textbf{X}),\textbf{Y})) = (\prod_{i=1}^{3}\mathcal{L}_i(f(\textbf{X}),\textbf{Y}))^{\frac{1}{3}}$, determines the aggregate performance of the loss functions. Owing to its multiplicative nature, the geometric mean can be sensitive to very low values among the loss functions, which can result in a lower aggregate value than the arithmetic mean, which accords each value with equal importance. This sensitivity grants the geometric mean with finer granularity, which is beneficial when: (1) the range and magnitude of values among the loss functions are comparable in scale, and (2) distinguishing solutions that exhibit exemplary performance on a subset of the loss functions is a priority. Table \ref{geometric_mean_table} includes the medians of the geometric means over all datasets for each of the comparative methods. CLML achieves the lowest (median) geometric mean.
    \begin{figure}[!t]
        \centering
    \begin{minipage}[t]{0.47\linewidth}
    \centering
    \captionof{table}{Medians of the geometric means for all methods across all datasets.}
    \label{geometric_mean_table}
    \vspace{3.5mm}
    \scalebox{0.8}{
    \begin{tabular}{lll}
    $f$ &Median $\mu_g(\boldsymbol{\mathcal{L}}(f))$\\ \hline
    GNB-BR& 0.481\\
    GNB-CC& 0.415\\
    ML$k$NN& 0.249\\
    C2AE& 0.394\\
    CLIF& 0.269\\
    DELA& 0.254\\
    CLML& \textbf{0.240}\\
    \end{tabular}}
    \end{minipage}
    \hspace{5pt}
    \begin{minipage}[t]{0.47\linewidth}
    \captionof{table}{Friedman statistic ($F$) and critical value ($C$) of each measure for all comparison methods across all datasets.}
    \label{friedman_table}
    \centering
    \scalebox{0.8}{
        \begin{tabular}{lrr}
    Metric &  $F$&  $C$ \\
    \midrule
  $\lambda(P(f))$&            2.88 &       \multirow{5}{*}{15.51} \\
  $\mu_g(\boldsymbol{\mathcal{L}}(f))$&           21.37 &\\
    $\mathcal{L}_1$&           35.62 &\\
  $\mathcal{L}_2$ &           25.64&\\
  $\mathcal{L}_3$&           17.75 &\\
    \end{tabular}}
    \end{minipage}
    \end{figure}

    \begin{figure*}[t]
    \centering
    \begin{minipage}{0.48\textwidth}
        \includegraphics[width=0.85\textwidth]{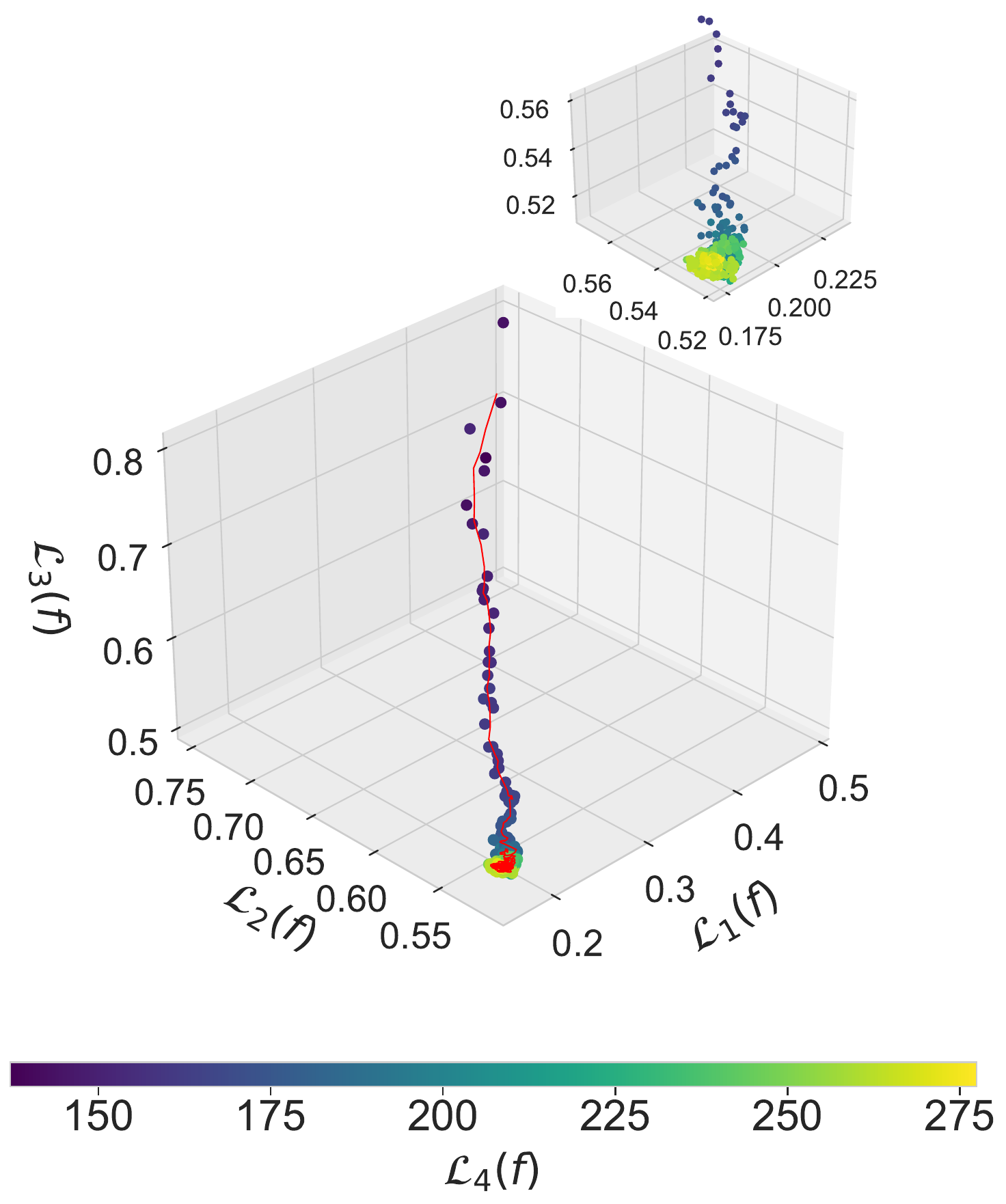}
        \centerline{(a) CAL500}
    \end{minipage}
    \begin{minipage}{0.48\textwidth}
        \includegraphics[width=0.85\textwidth]{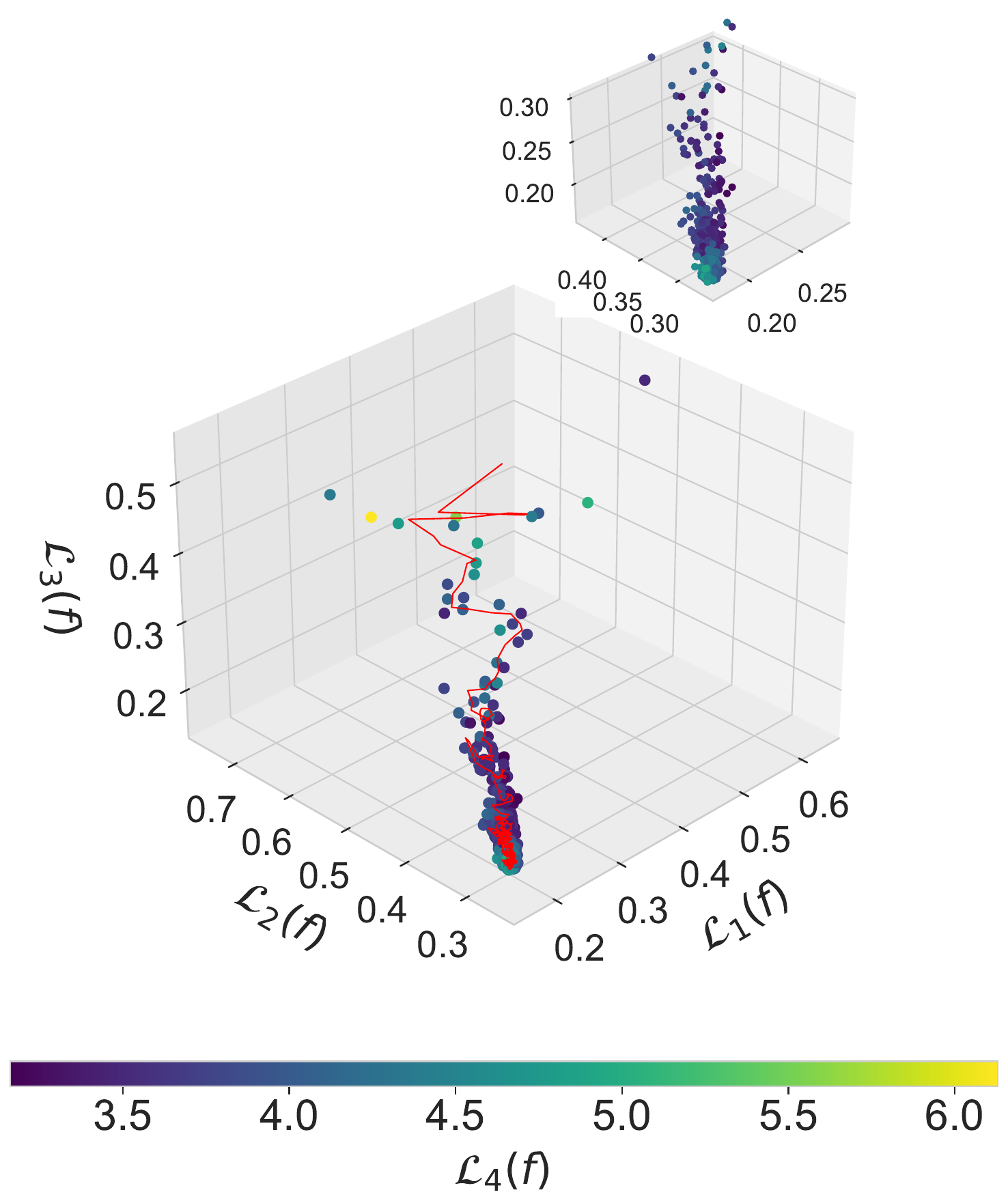}
        \centerline{(b) emotions}
    \end{minipage}
    \caption{The training curves of CLML plotted against $\mathcal{L}_1(f(\textbf{X}),\textbf{Y})$, $\mathcal{L}_2(f(\textbf{X}),\textbf{Y})$, and $\mathcal{L}_3(f(\textbf{X}),\textbf{Y})$. The colour represents the averaged binary cross-entropy loss $\mathcal{L}_4(f(\textbf{X}),\textbf{Y})$, which is tracked independently during the optimisation process. The red line shows the moving average trajectory of CLML. A zoom-in plot is presented at the top right of each subplot to highlight the area of convergence.}
    \label{training curves}
    \end{figure*}

    \subsubsection{Non-parametric tests}
    We employ the widely-used non-parametric Friedman test \citet{demvsar2006statistical} to measure any statistically significant performance differences between the methods concerning the Lebesgue contribution, the geometric mean, and each loss function. Table \ref{friedman_table} exhibits the significant differences between the geometric means and each loss function across all comparative methods and datasets. Interestingly, Table \ref{friedman_table} shows that the Friedman test of the Lebesgue contributions did not reflect these differences. Note that these observations do not indicate uniformity in the performance of all methods. For example, while different methods may contribute solutions of different qualities, said solutions may intersect in terms of their Lebesgue contributions. This is evident in the supplementary Tables \ref{tab:hv_contributions1} and \ref{tab:hv_contributions2}, where all dominated methods have a Lebesgue contribution of zero. 
    
    A further Bonferroni-Dunn test is carried out along with a critical difference analysis at $\alpha=0.05$ \citep{Dunn1961MultipleCA,zhang2004statisticala}. A pairwise comparison between the average ranks of each algorithm is performed, with CLML set as a control algorithm. The critical difference plots in Figure \ref{CDBonferroni} illustrate the rankings of each algorithm. The rankings are presented in ascending order, where the best method is on the leftmost side of the plot.
 
   In summary, CLML achieves the best aggregate rank of $2.50$ (aggregated among all measures and datasets presented in Figure \ref{CDBonferroni}), compared to $2.90$ of DELA ($+13.79\%$), $3.02$ of CLIF ($+17.22\%$), $3.9$ of ML$k$NN ($+35.89\%$), $4.54$ of C2AE ($+44.93\%$), $5.18$ of GNB-CC ($+51.74\%$), and $6.0$ of GNB-BR ($+58.33\%$). The only ranking that CLML is not the best at is concerning $\mathcal{L}_1$ and $\mathcal{L}_3$ (the former achieving a tie with CLIF, and the latter losing to CLIF by $0.1$). Furthermore, the Bonferroni-Dunn critical difference test is not powerful enough to detect a significant difference in rankings, even though CLML achieves a better rank in most cases. This largely stems from DELA and CLIF being the current state-of-the-art, remaining competitive with CLML in most cases, thus achieving a comparable ranking that is often within the critical difference. The significance of CLML's promising performance is accentuated by its simple feedforward representation that does not include embedding or perturbation techniques seen in DELA, CLIF, or C2AE. This highlights the influence and importance of a consistent loss function that properly directs optimisation behaviour and quality of multi-label learners. To further support this claim, the next section provides further analysis to examine the relationship between the average binary cross entropy and the loss functions that are optimised by CLML.
    
    \section{Effectiveness of the Lebesgue measure over surrogate loss}

    This section analyses the relationship between the loss functions: $\mathcal{L}_1$, $\mathcal{L}_2$, and $\mathcal{L}_3$, and the averaged binary cross entropy loss representing a standard surrogate loss function \citep{wu2017c2ae,bai2021mpvae,hang2022collaborative}, denoted $\mathcal{L}_4$.  The following key observations can be made.
    
    \subsection{Moving average trajectory} Figure \ref{training curves} plots the training trajectory of CLML concerning the four loss functions during training on the CAL500 and emotions datasets. The remaining training curves are available in Figures \ref{training_curves_supplement_1} and \ref{training_curves_supplement2} in Section \ref{surrogate_extended}. The red line traces the moving average trajectory of CLML on the approximate loss landscape defined by the three loss functions $\mathcal{L}_1$, $\mathcal{L}_2$, and $\mathcal{L}_3$, while the point colour denotes the $\mathcal{L}_4$ loss. The moving average trajectory indicates a distinct and consistent decline of all the desired loss functions $\mathcal{L}_1$, $\mathcal{L}_2$, and $\mathcal{L}_3$. Recall that CLML specifically optimises the Lebesgue \textit{contribution} $\lambda(P(f))$, \textit{i.e.}, the contribution toward the improvement of the Lebesgue measure $\lambda(H(F,R))$. In this case, by empirical observation, maximising the Lebesgue contribution directly corresponds to the minimisation of the desired loss functions. Furthermore, despite a degree of stochasticity, the overarching trend indicates that all three loss functions can be optimised in tandem despite their conflicting nature (see \citet{wu2020multi} for an analysis of the conflicting nature of multi-label loss functions). This is primarily due to the inherent nature of the Lebesgue measure, which naturally considers the trade-offs between the different multi-label loss functions. Hence, the optimisation behaviour of CLML naturally follows a path that understands and accounts for the best trade-offs between these loss functions. Furthermore, Figure \ref{training curves} draws attention to the increase in $\mathcal{L}_4$ near the point of convergence on both CAL500 and emotions. On both datasets, the minimisation of $\mathcal{L}_1$, $\mathcal{L}_2$, and $\mathcal{L}_3$ does not directly correspond to the minimisation of $\mathcal{L}_4$, which suggests a discrepancy between the surrogate and desired loss functions. 

    \begin{figure}[!h]
    \centering
    \begin{minipage}{0.48\textwidth}
    \centering
    \includegraphics[width=0.8\textwidth,  trim = 0 0 0 2cm, clip]{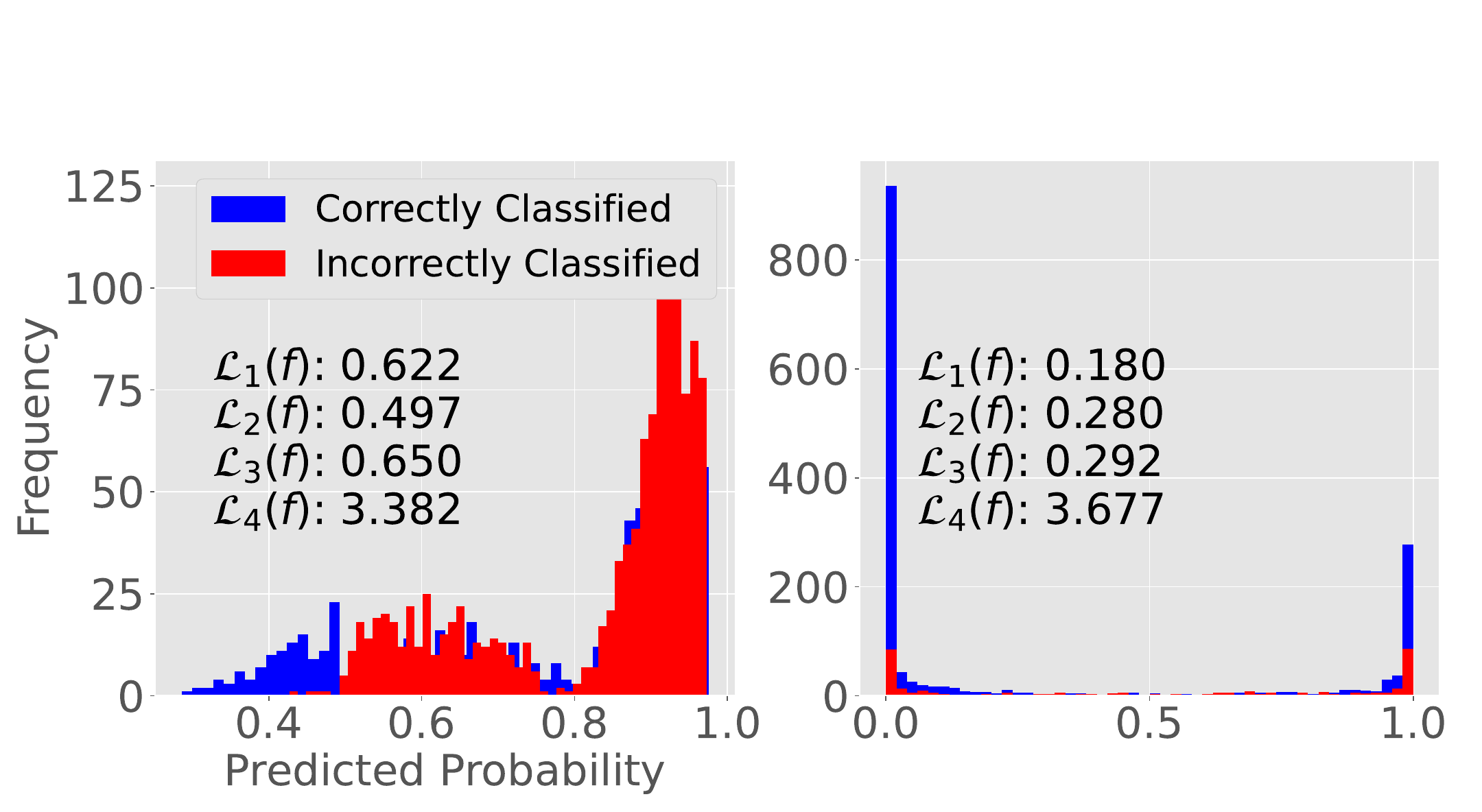}
        \centerline{(a) Label distributions (before and after training).}
    \end{minipage}
    \begin{minipage}{0.48\textwidth}
    \centering
    \includegraphics[width=0.8\textwidth, trim = 0 0 0 1cm, clip]{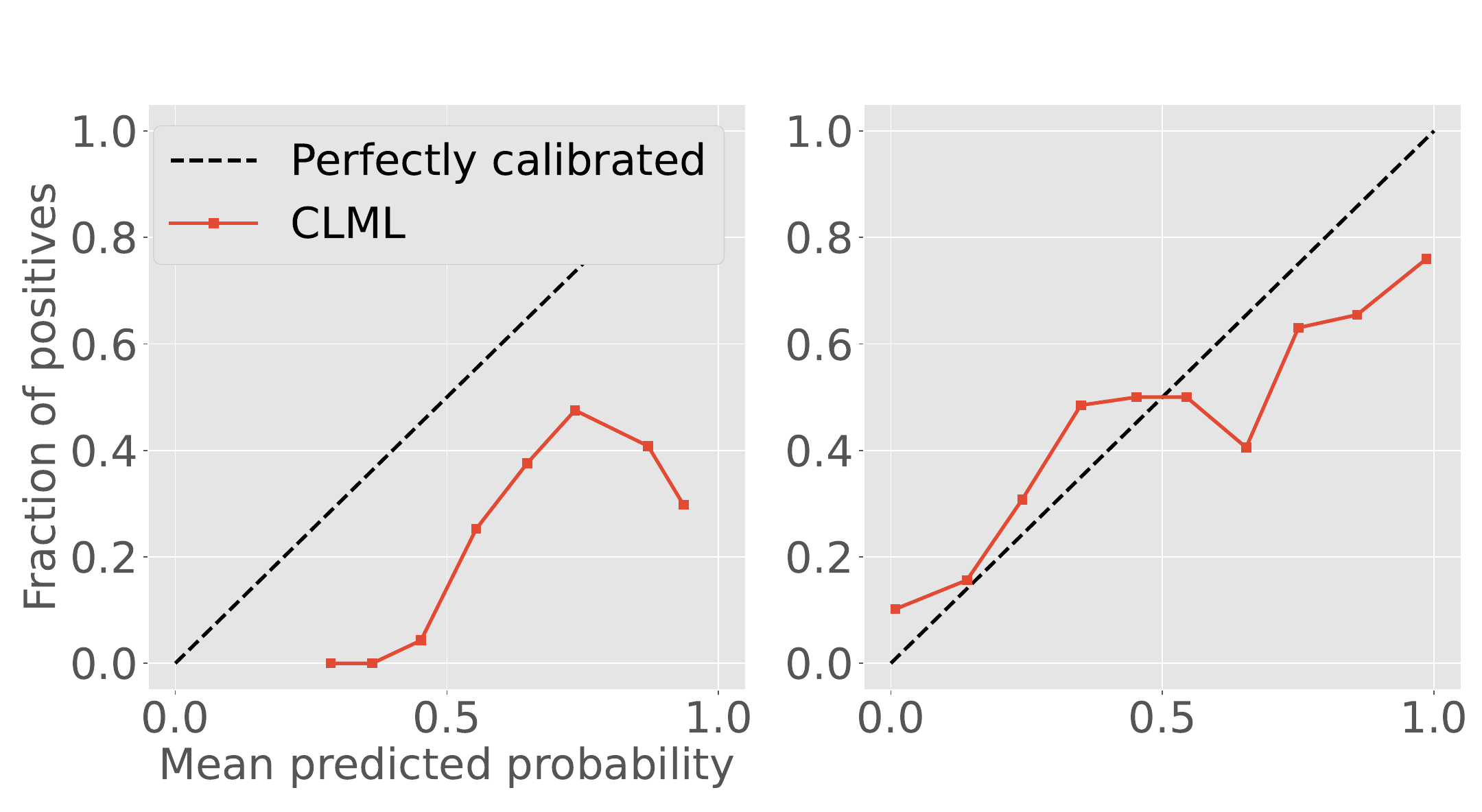}
        \centerline{(b) Calibration curves (before and after training).}
    \end{minipage}
    \caption{Results of CLML's incumbent solution shown on the emotions dataset before (after 1 epoch) and after training (all epoch). The number of bins for the distributions and the calibration plots are $B_D=50$ and $B_C=10$, respectively.}
    \label{confidence}
    \vspace{-3mm}
    \end{figure}

    \subsection{Label confidence vs. label accuracy} 
    Figures \ref{confidence}(a) and (b) plot the distributions of CLML's incumbent label probabilities and the calibration curve of CLML's incumbent solution respectively, after one epoch (left) and after training (right). Both correct and incorrect label probability distributions are uni-modal and share a similar shape, with noticeable variation in confidence ranging from $0.5$ to $0.8$. This suggests that CLML is initially less confident in its prediction of the absence of labels than its prediction of the presence of labels (irrespective of correctness). After training, the distributions are heavily skewed and bi-modal, which suggests that CLML is very confident in predicting labels. This bi-modal shape applies to both the correct and incorrect label distributions. Notably, the value of $\mathcal{L}_4$ after one epoch increases from $3.382$ to $3.677$ after training, an undesired effect. This increase can be attributed to CLML's increased degree of confidence in the incorrectly classified labels. However, $\mathcal{L}_1$, $\mathcal{L}_2$, and $\mathcal{L}_3$ are desirably lower after training, which is further supported by the improvement in the calibration curve of CLML after training in Figure\ref{confidence}(b). These observations highlight the discrepancy between confidence and accuracy, which underscores the importance of directly handling accuracy (without surrogacy) in multi-label learning.
	
	\section{Conclusions}

  Deep learning techniques have advanced the field of multi-label learning, although inconsistencies remain between surrogate loss functions and multi-label functions and have seldom been addressed. This motivates the design of the Consistent Lebesgue Measure-based Multi-label Learner (CLML) to optimise multiple non-convex and discontinuous loss functions simultaneously using a novel Lebesgue measure-based learning objective. By analysis, we proved that CLML is theoretically consistent with the underlying loss functions that are optimised. Furthermore, empirical evidence supports our theory by demonstrating a $13.79\%$ to $58.33\%$ improvement in the critical difference rankings. These results are especially significant due to the \textit{simplicity} of CLML, which achieves state-of-the-art results without the need to explicitly consider label interactions via label graphs, latent semantic embeddings, or perturbation-based conditioning. Lastly, our analysis shows that CLML can naturally account for the best trade-offs between multiple multi-label loss functions that are known to exhibit conflicting behaviour. CLML's state-of-the-art performance further highlights the importance of optimisation consistency over model complexity. Thus, the findings of this paper analytically emphasises the overall significance of consistency and goal alignment in multi-label learning. 
  
\newpage

\bibliography{main}
\bibliographystyle{icml2024}

%%%%%%%%%%%%%%%%%%%%%%%%%%%%%%%%%%%%%%%%%%%%%%%%%%%%%%%%%%%%%%%%%%%%%%%%%%%%%%%
%%%%%%%%%%%%%%%%%%%%%%%%%%%%%%%%%%%%%%%%%%%%%%%%%%%%%%%%%%%%%%%%%%%%%%%%%%%%%%%
% APPENDIX
%%%%%%%%%%%%%%%%%%%%%%%%%%%%%%%%%%%%%%%%%%%%%%%%%%%%%%%%%%%%%%%%%%%%%%%%%%%%%%%
%%%%%%%%%%%%%%%%%%%%%%%%%%%%%%%%%%%%%%%%%%%%%%%%%%%%%%%%%%%%%%%%%%%%%%%%%%%%%%%
\newpage
\appendix
\onecolumn
\section{Optimisation 
    Process}\label{optimisationprocess}
    The overall optimisation procedure using covariance matrix adaptation and the Lebesgue measure is described clearly in the following sections.

    \subsection{Covariance matrix adaptation}\label{covariancematrixadaptation}

    Co-variance Matrix Adaptation Evolutionary Strategy (CMA-ES) \cite{hansen2001completely} is a gradient-free numerical optimisation technique well-suited for non-convex and non-differentiable optimisation problems. Suppose the representation of a learner $f$ can be denoted as a vector consisting of its learnable parameters $\theta^f$. CMA-ES works by sampling $\lambda$ solutions from a multi-variate normal distribution as follows:
	\begin{equation}
		\theta^f_i\sim\textbf{m}+\sigma\mathcal{N}_i(0,\textbf{C})\quad \forall i, \quad 1\leq i \leq \lambda
	\end{equation}
	where $\theta^f_i$ is the learnable parameters of the $i^{th}$ learner, $\textbf{m}$ is a mean-vector which represents the expected density of parameters of $f$, $\sigma$ the step-size, and \textbf{C} the covariance matrix. CMA-ES therefore iteratively updates $\textbf{m}$ and $\textbf{C}$ via the following:
	\begin{align}
		\textbf{m}^{t+1} = \textbf{m}^t+ \sigma \sum_{i=1}^{\mu}w_i \theta^{f^{top}}_i&\\
		\textbf{C}^{t+1}=(1-c_{cov})\textbf{C}^t+c_{cov}\sum_{i=1}^{\mu}w_i \theta^{f^{top}}_i(\sum_{i=1}^{\mu}w_i \theta^{f^{top}}_i)^T&\\
	\end{align}
	where $c_{cov}$ is the learning rate,  $\sum_{i=1}^{\mu}w_i\theta^{f^{top}}_i$ is the weighted sum of the $\mu$-highest ranked solutions at time $t$, where the weights $w_1>w_2>\cdots>w_\mu>0$ and $\sum_{i=1}^{\mu}w_i=1$. It is also deemed that solutions $\theta^{f^{top}}_i\sim \textbf{m}+\sigma\mathcal{N}_i(0,\textbf{C})$ are ranked such that $\theta^{f^{top}}_1\prec\cdots\prec\theta^{f^{top}}_\mu$ and that the $\mu$ ranked solutions are a subset of the total number of sampled solutions, i.e. $\mu < \lambda$. This method is referred to as rank-one update.

    \begin{algorithm*}[!h]
    \caption{Consistent Lebesgue Measure-based Multi-label Learner}
    \label{optimisationalgorithm}
    \begin{algorithmic}
    \STATE {\bfseries Input:} {Initial covariance matrix $\textbf{C}^0$, and density vector $\textbf{m}^0$, and maximum number of epoch $T$};
    \STATE Initialise $R^0$ to unit vector $\{1\}^3$;
    \STATE Initialise $F^0 = \{\emptyset\}$;
    \STATE Set $t=0$;
    \STATE Set $f^0 = \textbf{m}^0$;
    \WHILE{$t < T$}{
        \STATE Generate $f^i\sim \textbf{m}^t + \sigma\mathcal{N}_i(0,\textbf{C}^t)$, $1\leq i \leq \lambda$;
        \STATE Set $F^{t+1} = \bigcup_{i=1}^\lambda \{f^i\}$;
        \FOR{$f^{i} \in F^{t+1}$}{
            \STATE Calculate the training ($tra$) and validation ($val$) loss values for each loss function: $\mathcal{L}_1(f^{i})$, $\mathcal{L}_2(f^{i})$, and $\mathcal{L}_3(f^{i})$;
            \STATE Estimate $\lambda(P(f^i))$ using the Monte Carlo method over loss functions $\mathcal{L}_1^{tra}(f^i)$, $\mathcal{L}_2^{tra}(f^i)$, and $\mathcal{L}_3^{tra}(f^i)$, and prescribe it as the fitness for $f^i$;
            \STATE Archive the loss values $\mathcal{L}_1^{val}(f^i)$, $\mathcal{L}_2^{val}(f^i)$, and $\mathcal{L}_3^{val}(f^i)$, and corresponding function $f^i$;
        }
        \ENDFOR
        \STATE Update density $\textbf{m}^{t+1}$ to the average of the newly generated solutions $\forall{f^{i}}\in F^{t+1}$;
        Update $\textbf{C}^{t+1}$ via rank-one method using the prescribed $\lambda(P(f^i))$ as fitness values $\forall{f^{i}}\in F^{t+1}$;
        \STATE Update $R^{t+1}$ by calculating the mutually non-dominated solutions in $R^{t}\cup F^{t+1}$;
        \STATE Set $f^{t+1}$ to the best solution in $F^{t+1}$ according to its prescribed fitness value;
    }
    \ENDWHILE
    \STATE {\bfseries Return:} {Incumbent solutions for each loss function from archive: $\mathcal{L}_1^{val}(f^i)$, $\mathcal{L}_2^{val}(f^j)$, and $\mathcal{L}_3^{val}(f^k)$, and the final incumbent solution $f^T$};
    \end{algorithmic}
    \end{algorithm*}

    \begin{table*}[t]
	\centering
	\caption{Summary of datasets used in this paper. $D$, $N$, and $K$ correspond to the number of features, instances, and labels, respectively.}
	\scalebox{0.9}{
		\begin{tabular}{cccccccc}	Dataset&$N$&$D$&$K$&$DK$&$K^\mu$&$DK/K^\mu$&$DK^\mu$\\
			\hline
			flags&194&19&7&133&3.392&39.21&64.45\\
			CAL500&502&68&174&$251,000$&26.044&$9,637.54$&$1,770.99$\\
			emotions&593&72&6&432&1.869&231.14&$134.57$\\
            genbase&662&1186&27&$32,022$&1.252&$25,576.68$&$1,484.87$\\
            enron&1702&1001&53&$53,053$&3.378&$15,705.45$&$3,381.38$\\
			yeast&2417&103&14&1442&4.237&340.335&$436.411$\\
			%scene&294&2407&6&1764&1.074&$1,642.46$&$315.756$\\
			%corel5k&499&5000&374&$186,626$&3.522&$52,988.64$&$1,757.48$\\
			%medical&1449&978&45&$65,205$&1.245&$52,373.49$&$1,804.01$\\
            tmc2007-500&28,596&500&22&11,000&2.158&5,097.31&1,079\\
            mediamill&43,907&120&101&12,120&4.376&2,769.65&525.12\\
            IMDB-F&120,900&1001&28&28,028&2.000&14,014&2,002\\
	\end{tabular}}
	\label{tab_dataset}
    \end{table*}

    \subsection{Lebesgue measure estimation using Monte Carlo sampling}\label{lebesguemeasureestimation}

    The Lebesgue measure $\lambda(H(F,R))$ described in Eq. \ref{lebesguemeasure} integrates the area covered by a set of loss function vectors in a multi-dimensional objective space. This measure is comprised of three sets: $F$, $R$, and $Z$. $F$ denotes the set of representations of functions (which map the input data to a vector of loss function values). $R$ denotes the set of mutually non-dominating loss vectors. Initially, $R$ is set to the unit loss vector $\{1\}^3$, which denotes the worst possible performance for Hamming-loss, one minus the Micro-$F_1$, and one minus the label ranking average precision. Last, $Z$ denotes the set containing all possible loss function vectors in the applicable multi-dimensional loss objective space.

    The Lebesgue contribution $\lambda(P(f))$ of a function $f$ measures the new marginal improvement of a function's loss vector over a set of previous loss vectors. In this paper, we use the Lebesgue contribution to quantify candidate functions found by CLML during the optimisation process. However, to efficiently calculate the Lebesgue contribution (especially when the set of functions $F$ and $R$ are sparsely populated during the early stages of the optimisation), we estimate the Lebesgue measure using Monte Carlo sampling. First, a sampling space $S \subseteq Z$ is defined that entirely contains $P(f)$, \textit{i.e.,} $P(f) \subseteq S \subseteq Z$. The sampling space can be problem-specific, however, in this paper, it is defined to contain all possible loss vectors between $\{0\}^3$ and $\{1\}^3$. Following, $g$ samples are drawn from $s_i\in S$ randomly and with uniform probability. Given $\{s_1,\cdots,s_g\}$, the Lebesgue contribution is estimated via $\hat{\lambda}(P(f))$ via the following:

    \begin{equation}
        \hat{\lambda}(P(f)) = \lambda(S(f)) = \frac{|\{s_i| s_i \in P(f)\}|}{g}\label{montecarlosampling}
    \end{equation}
    where $|\{s_i| s_i \in P(f)\}|$ is denoted as the number of randomly sampled solutions that exist in $P(f)$, also known as \textit{hits}. The probability $p$ of a sample being \textit{hit} is i.i.d. Bernoulli distributed, therefore, $\hat{\lambda}(P(f))$ converges to $\lambda(P(f))$ with $\frac{1}{\sqrt{pg}}$ \cite{laplace1814theorie}. 
    
    \subsection{Optimisation of the Lebesgue measure using covariance matrix adaptation}\label{optimisationprocedure}

    The optimisation process is described in Algorithm \ref{optimisationalgorithm}. Starting with an initial covariance matrix and density vector, CLML optimises the Lebesgue contribution of newly generated candidate functions obtained by perturbing a density vector and covariance matrix. Each function is evaluated using the Hamming-loss ($\mathcal{L}_1$), one minus the label ranking average precision ($\mathcal{L}_2$), one minus the micro-$F_1$ ($\mathcal{L}_3$). The density vector is updated toward the density of the current solutions, and the covariance matrix is updated using a rank-one method. CLML maintains an archive of solutions based on the validation fitness values that are derived from the prescribed loss functions. Ultimately, CLML returns the incumbent solutions (in terms of validation loss) for each of the loss functions from the archive and the final incumbent solution.
    
    \section{Experimental Protocol}\label{experiments}
    
    We conduct the experiments on nine widely-used multi-label datasets, shown in Table \ref{tab_dataset}. $K^\mu$ (the cardinality) of an instance measures the average number of associated class labels; $DK/K^\mu$, the theoretical maximum complexity of an instance, (\textit{i.e.}, the instance-level average dispersion of feature to label interactions); and $DK^\mu$, the average feature to label interactions of an instance. There are two important cases to consider. First, if dispersion is less than the average interaction, \textit{i.e.}, $DK/K^\mu < DK^\mu$, then the dataset contains high concentrations of rich instance-level feature-to-label interactions that are not apparent when examining the dataset as a whole. This can indicate that there are clusters of instances that share similar feature-to-label interactions, and therefore a less diverse dispersion of the possible feature-to-label interactions. Second, if dispersion is higher than the average interaction, \textit{i.e.}, $DK/K^\mu > DK^\mu$, the dataset as a whole has a greater expression of feature-to-label interactions than a given individual instance. Put differently, the dataset's instances each contain a subset of the total dataset interactions. The latter case is particularly challenging as it indicates a high number and variability of potential patterns and interactions between features and labels. The first case occurs in both Flags and Yeast and the second case occurs in the remaining datasets. 
    
    For each dataset, 30\% are partitioned to the test set \citep{sechidis2011stratification}. The remaining 70\% is further split such that 20\% is used as a validation set, and the remaining is used for training. We apply normalisation to all numerical features before training.
    \section{Ablation Study}\label{ablation}
    We trial the embedding dimension $C$ at eight separate values. It is important to note that the latent space does not need to express spatial relationships of tabulated data, hence the embedding dimension can be quite small (in contrast to computer vision in works such as \citet{gong2022contrastive}). In addition to $\mathcal{L}_1, \mathcal{L}_2$, and $\mathcal{L}_3$, we set $\mathcal{L}_4$ as the averaged binary cross-entropy loss and track its progress during optimisation. For each experiment, we set $\mathcal{O}=750$ (the maximum number of epochs). Here, we present the results for each of the embedding dimensions. 
    
    Figures \ref{lebesgue_plot1} and \ref{lebesgue_plot2} plot the Lebesgue measure of the sequence of functions obtained by CLML as $n\to \mathcal{O}$ (\textit{i.e.}, the archive of non-dominated solutions obtained by CLML in $\mathcal{O}$ epoch). Smaller embedding dimensions (\textit{i.e.}, $C\leq 80$) result in the best validation scores of $\lambda(H(F,R))$. To exemplify this, we tabulate the incumbent solution of the function sequence in terms of its $\mathcal{L}_1$, $\mathcal{L}_2$, and $\mathcal{L}_3$ scores on the validation set, against $\mathcal{L}_4$ according to each (non-normalised) value of $C$ in Table \ref{discrepancy_table1} and \ref{discrepancy_table2}. When $C=20$, we observe the lowest $\mathcal{L}_1$, $\mathcal{L}_2$, and $\mathcal{L}_3$ validation loss scores on the emotions dataset, and the lowest $\lambda(H(F,R))$ score on the CAL500 dataset. This observation indicates that CLML converges toward a better approximation of the Bayes predictors of $\mathcal{L}_1$, $\mathcal{L}_2$, and $\mathcal{L}_3$ on the emotions dataset, while on CAL500, CLML finds functions with more desirable trade-offs between the variant loss functions, hence the higher Lebesgue measure. These values motivate our recommendation to set the number of embedding dimensions to $C=20$.

    \begin{minipage}{0.48\textwidth}
        \includegraphics[width=0.9\textwidth]{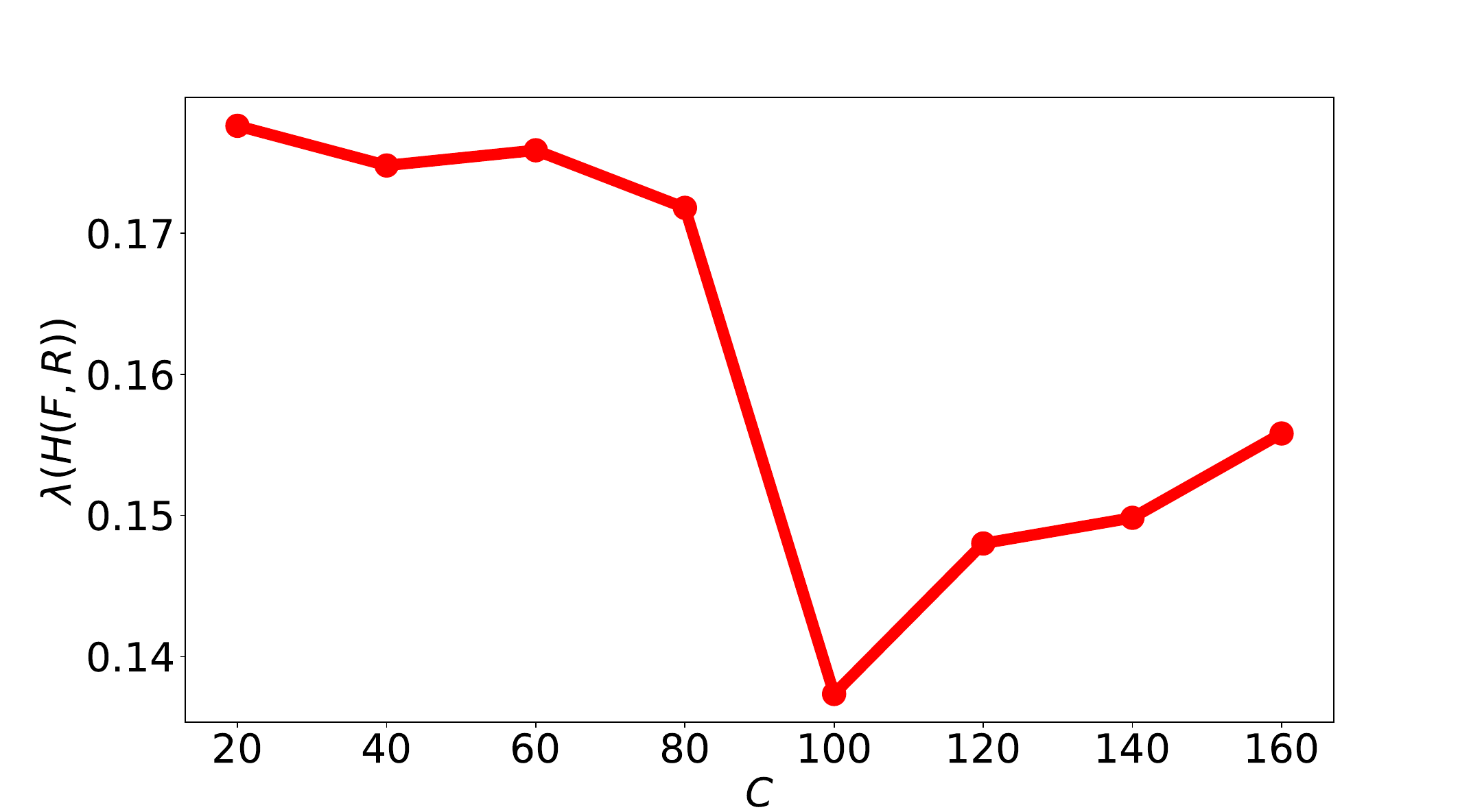}
        \captionof{figure}{The best Lebesgue measure obtained on CAL500 at each embedding dimension of the sequence of function sets $\underset{n\to \mathcal{O}}{\text{lim}}\lambda(H(F^{(n)},R))$.}
        \label{lebesgue_plot1}
    \end{minipage}
    \hfill
    \begin{minipage}{0.48\textwidth}
        \centering
        \captionof{table}{Best validation loss values of the incumbent solution for each embedding dimension on CAL500.}
        \label{discrepancy_table1}
        \begin{tabular}{ccccc}
        $C$ & $\mathcal{L}_1$ &  $\mathcal{L}_2$ &   $\mathcal{L}_3$ &  $\mathcal{L}_4$ \\
        \hline
          20.0 &  0.169 &  0.523 &  \textbf{0.509} &    \textbf{138.068} \\
          40.0 &  0.171 &  \textbf{0.522} &  0.518 &    143.809 \\
          60.0 &  0.169 &  0.523 &  0.520 &    144.201 \\
          80.0 &  \textbf{0.161} &  0.527 &  0.525 &    149.604 \\
         100.0 &  0.196 &  0.529 &  0.534 &    157.877 \\
         120.0 &  0.171 &  0.529 &  0.539 &    155.555 \\
         140.0 &  0.167 &  0.528 &  0.534 &    151.250 \\
         160.0 &  0.168 &  0.526 &  0.533 &    153.533 \\
        \end{tabular}
    \end{minipage}

     \begin{center}
    \begin{minipage}{0.48\textwidth}
        \includegraphics[width=0.9\textwidth]{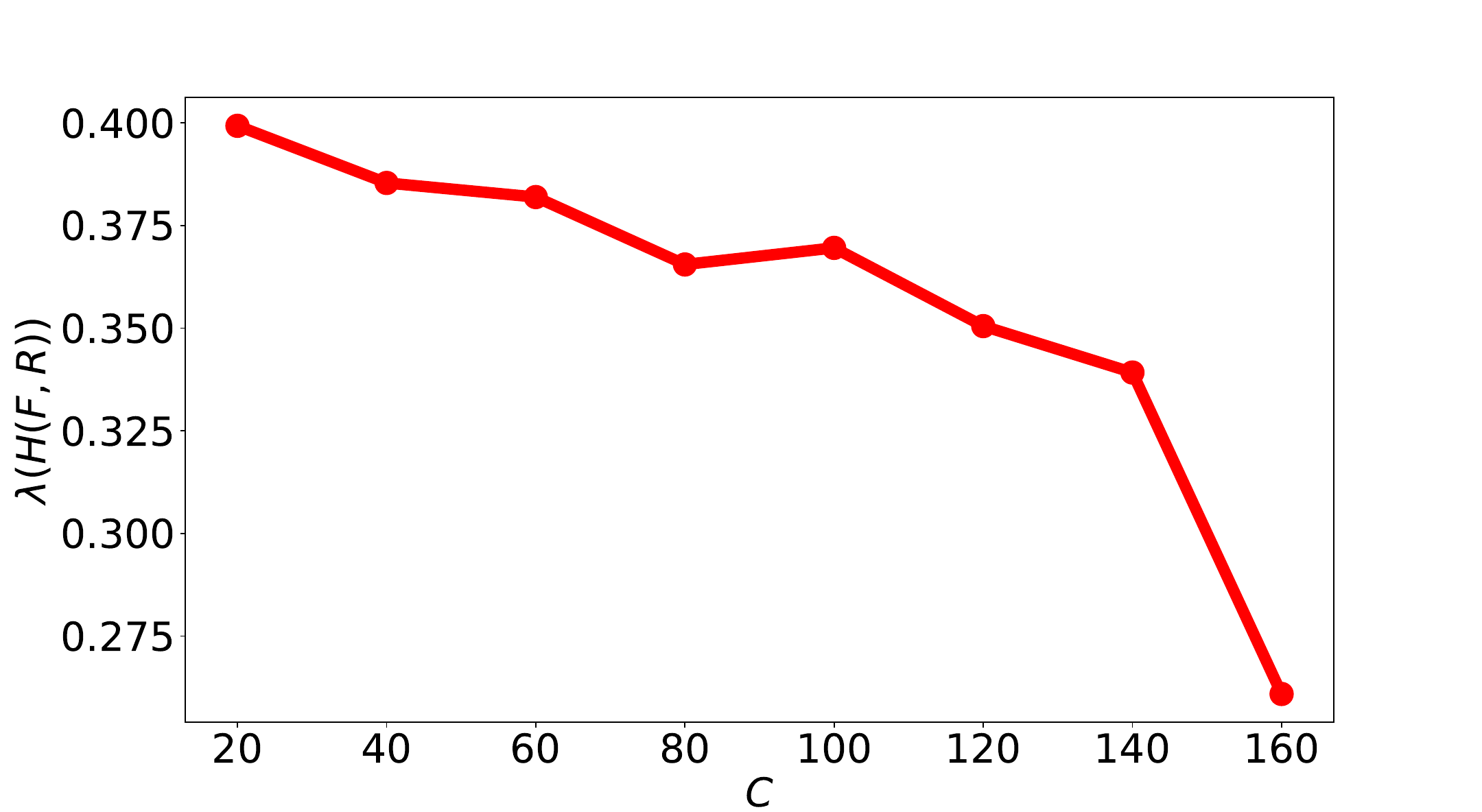}
        \captionof{figure}{The best Lebesgue measure obtained on emotions at each embedding dimension of the sequence of function sets $\underset{n\to \mathcal{O}}{\text{lim}}\lambda(H(F^{(n)},R))$.}
        \label{lebesgue_plot2}
    \end{minipage}
    \hfill
    \begin{minipage}{0.48\textwidth}
        \centering
        \captionof{table}{Best validation loss values of the incumbent solution for each embedding dimension on emotions.}
        \label{discrepancy_table2}
        \begin{tabular}{ccccc}
        $C$ & $\mathcal{L}_1$ &  $\mathcal{L}_2$ &   $\mathcal{L}_3$ &  $\mathcal{L}_4$ \\
        \hline
          20.0 &  \textbf{0.187} &  \textbf{0.283} &  \textbf{0.178} &      3.399 \\
          40.0 &  0.199 &  0.307 &  0.197 &      \textbf{3.246} \\
          60.0 &  0.192 &  0.299 &  0.196 &      3.255 \\
          80.0 &  0.196 &  0.306 &  0.196 &      3.910 \\
         100.0 &  0.199 &  0.302 &  0.199 &      3.742 \\
         120.0 &  0.210 &  0.333 &  0.212 &      3.557 \\
         140.0 &  0.202 &  0.313 &  0.193 &      4.380 \\
         160.0 &  0.250 &  0.398 &  0.252 &      5.758 \\
        \end{tabular}
    \end{minipage}
    \end{center}

    \begin{table*}[h!]
\centering
\caption{Lebesgue measure contributions of each $f$ on datasets: CAL500 to genbase}
\begin{scalebox}{0.75}{
\begin{tabular}{l|c|c|c|c|c}
Dataset &                      Method &               Solution & HV Contribution & Normalized Contribution & Geometric Mean \\
\hline
  CAL500 &  GNB-BR &  (0.547, 0.713, 0.647) &               0 &                       0 &       0.631976 \\
  CAL500 &  GNB-CC &  (0.255, 0.633, 0.741) &               0 &                       0 &       0.492774 \\
  CAL500 &                       MLKNN &  (0.150, 0.637, 0.554) &        0.000351 &                0.020488 &       0.375585 \\
  CAL500 &                        C2AE &  (0.258, 0.536, 0.534) &               0 &                       0 &       0.419409 \\
  CAL500 &                        CLIF &  (0.137, 0.681, 0.502) &        0.007068 &                0.412185 &       0.360752 \\
  CAL500 &                        DELA &  (0.171, 0.633, 0.596) &               0 &                       0 &       0.400750 \\
  CAL500 &                    CLML &  (0.168, 0.526, 0.520) &        0.009729 &                0.567326 &       0.358231 \\
   yeast &  GNB-BR &  (0.319, 0.472, 0.351) &               0 &                       0 &       0.375014 \\
   yeast &  GNB-CC &  (0.319, 0.481, 0.415) &               0 &                       0 &       0.399048 \\
   yeast &                       MLKNN &  (0.213, 0.375, 0.298) &               0 &                       0 &       0.287821 \\
   yeast &                        C2AE &  (0.221, 0.358, 0.272) &        0.003081 &                0.425447 &       0.278355 \\
   yeast &                        CLIF &  (0.227, 0.391, 0.275) &               0 &                       0 &       0.290108 \\
   yeast &                        DELA &  (0.226, 0.391, 0.276) &               0 &                       0 &       0.289957 \\
   yeast &                    CLML &  (0.211, 0.364, 0.266) &        0.004160 &                0.574553 &       0.273480 \\
   enron &  GNB-BR &  (0.198, 0.725, 0.776) &               0 &                       0 &       0.481206 \\
   enron &  GNB-CC &  (0.125, 0.638, 0.742) &               0 &                       0 &       0.389782 \\
   enron &                       MLKNN &  (0.056, 0.529, 0.436) &               0 &                       0 &       0.234964 \\
   enron &                        C2AE &  (0.189, 0.665, 0.487) &               0 &                       0 &       0.393941 \\
   enron &                        CLIF &  (0.053, 0.499, 0.381) &        0.002600 &                0.449837 &       0.216576 \\
   enron &                        DELA &  (0.054, 0.493, 0.386) &        0.000126 &                0.021742 &       0.218104 \\
   enron &                    CLML &  (0.054, 0.488, 0.411) &        0.003055 &                0.528421 &       0.220966 \\
 genbase &  GNB-BR &  (0.052, 0.479, 0.538) &               0 &                       0 &       0.237314 \\
 genbase &  GNB-CC &  (0.008, 0.078, 0.091) &               0 &                       0 &       0.037745 \\
 genbase &                       MLKNN &  (0.033, 0.454, 0.331) &               0 &                       0 &       0.170749 \\
 genbase &                        C2AE &  (0.345, 0.823, 0.561) &               0 &                       0 &       0.542500 \\
 genbase &                        CLIF &  (0.046, 0.793, 0.539) &               0 &                       0 &       0.269161 \\
 genbase &                        DELA &  (0.002, 0.020, 0.001) &        0.144566 &                1.000000 &       0.003138 \\
 genbase &                    CLML &  (0.020, 0.239, 0.117) &               0 &                       0 &       0.082065 \\
\end{tabular}
}\end{scalebox}
\label{tab:hv_contributions1}
\end{table*}
\begin{table*}[!h]
\centering
\caption{Lebesgue measure contributions of each $f$ on datasets: emotions to mediamill}
\begin{scalebox}{0.75}{
\begin{tabular}{l|c|c|c|c|c}
     Dataset &                      Method &               Solution & HV Contribution & Normalized Contribution & Geometric Mean \\ \hline
    emotions &  GNB-BR &  (0.410, 0.458, 0.271) &               0 &                       0 &       0.370604 \\
    emotions &  GNB-CC &  (0.265, 0.383, 0.256) &               0 &                       0 &       0.295937 \\
    emotions &                       MLKNN &  (0.268, 0.497, 0.361) &               0 &                       0 &       0.363696 \\
    emotions &                        C2AE &  (0.537, 0.556, 0.488) &               0 &                       0 &       0.526199 \\
    emotions &                        CLIF &  (0.223, 0.412, 0.246) &               0 &                       0 &       0.282547 \\
    emotions &                        DELA &  (0.216, 0.353, 0.214) &        0.005072 &                0.191264 &       0.253682 \\
    emotions &                    CLML &  (0.205, 0.328, 0.224) &        0.021444 &                0.808736 &       0.246669 \\
       flags &  GNB-BR &  (0.443, 0.560, 0.439) &               0 &                       0 &       0.477465 \\
       flags &  GNB-CC &  (0.402, 0.496, 0.360) &               0 &                       0 &       0.415483 \\
       flags &                       MLKNN &  (0.307, 0.302, 0.233) &               0 &                       0 &       0.278388 \\
       flags &                        C2AE &  (1.000, 1.000, 1.000) &               0 &                       0 &       1.000000 \\
       flags &                        CLIF &  (0.298, 0.316, 0.217) &               0 &                       0 &       0.273610 \\
       flags &                        DELA &  (0.271, 0.284, 0.231) &        0.006058 &                0.463227 &       0.260929 \\
       flags &                    CLML &  (0.281, 0.285, 0.205) &        0.007020 &                0.536773 &       0.254035 \\
     IMDB-F &  GNB-BR &  (0.276, 0.875, 0.489) &               0 &                       0 &       0.490350 \\
      IMDB-F &  GNB-CC &  (0.391, 0.892, 0.506) &               0 &                       0 &       0.560657 \\
      IMDB-F &                       MLKNN &  (0.044, 0.929, 0.376) &        0.000180 &                0.003999 &       0.249118 \\
      IMDB-F &                        C2AE &  (0.052, 0.743, 0.324) &        0.044679 &                0.993734 &       0.231745 \\
      IMDB-F &                        CLIF &  (0.049, 0.825, 0.391) &               0 &                       0 &       0.250780 \\
      IMDB-F &                        DELA &  (0.054, 0.831, 0.381) &               0 &                       0 &       0.257251 \\
      IMDB-F &                    CLML &  (0.048, 0.802, 0.358) &        0.000102 &                0.002268 &       0.240283 \\
 tmc2007-500 &  GNB-BR &  (0.598, 0.759, 0.822) &               0 &                       0 &       0.719892 \\
 tmc2007-500 &  GNB-CC &  (0.413, 0.697, 0.784) &               0 &                       0 &       0.608986 \\
 tmc2007-500 &                       MLKNN &  (0.065, 0.351, 0.261) &               0 &                       0 &       0.181723 \\
 tmc2007-500 &                        C2AE &  (0.051, 0.250, 0.145) &               0 &                       0 &       0.122771 \\
 tmc2007-500 &                        CLIF &  (0.040, 0.203, 0.123) &        0.007761 &                1.000000 &       0.099828 \\
 tmc2007-500 &                        DELA &  (0.041, 0.207, 0.128) &               0 &                       0 &       0.102543 \\
 tmc2007-500 &                    CLML &  (0.080, 0.427, 0.321) &               0 &                       0 &       0.222227 \\
   mediamill &  GNB-BR &  (0.338, 0.845, 0.787) &               0 &                       0 &       0.608021 \\
   mediamill &  GNB-CC &  (0.130, 0.708, 0.786) &               0 &                       0 &       0.416542 \\
   mediamill &                       MLKNN &  (0.030, 0.412, 0.283) &               0 &                       0 &       0.152028 \\
   mediamill &                        C2AE &  (0.042, 0.445, 0.285) &               0 &                       0 &       0.174138 \\
   mediamill &                        CLIF &  (0.027, 0.364, 0.216) &        0.035504 &                1.000000 &       0.128975 \\
   mediamill &                        DELA &  (0.031, 0.380, 0.252) &               0 &                       0 &       0.144388 \\
   mediamill &                    CLML &  (0.035, 0.464, 0.353) &               0 &                       0 &       0.178885 \\
\end{tabular}
}\end{scalebox}
\label{tab:hv_contributions2}
\end{table*}

    \section{Extended evaluation of multi-label classification performances}

    Tables \ref{tab:hv_contributions1} and \ref{tab:hv_contributions2} show the expanded view of the loss values ($\mathcal{L}_1$, $\mathcal{L}_2$, and $\mathcal{L}_3$), the Lebesgue contribution ($\lambda(P(f))$), the normalised Lebesgue contribution, and geometric means of each comparative method on each dataset. A zero value on the Lebesgue contribution indicates that a given function is dominated by all other functions on the given dataset, \textit{i.e.}, it does not contribute toward the improvement of the volume over $\boldsymbol{\mathcal{L}}(f(\textbf{X}),\textbf{Y})$.

    \begin{figure*}
    \centering
    \begin{minipage}{0.48\textwidth}
        \includegraphics[width=0.85\textwidth]{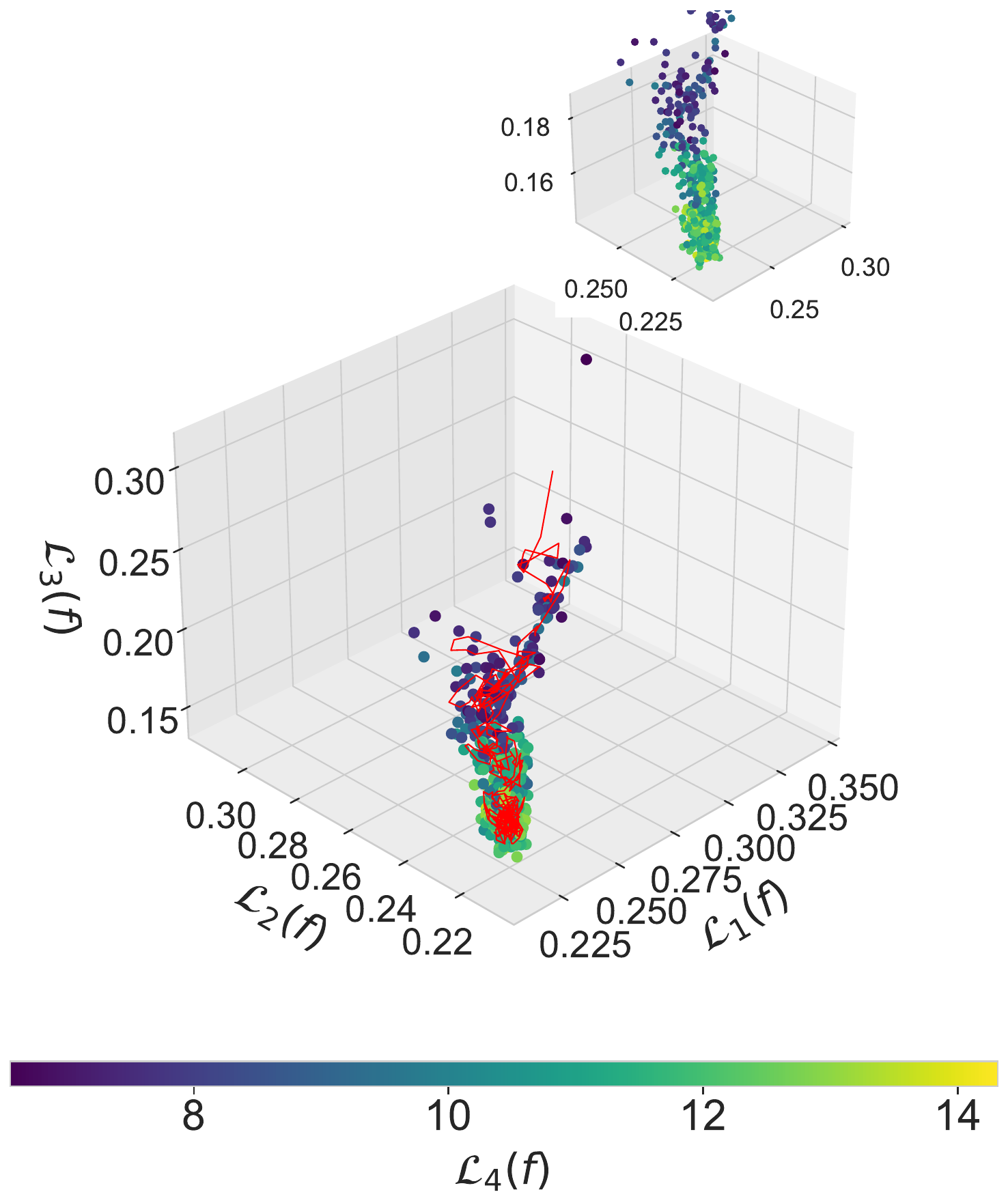}
        \centerline{(c) flags}
    \end{minipage}
    \begin{minipage}{0.48\textwidth}
        \includegraphics[width=0.85\textwidth]{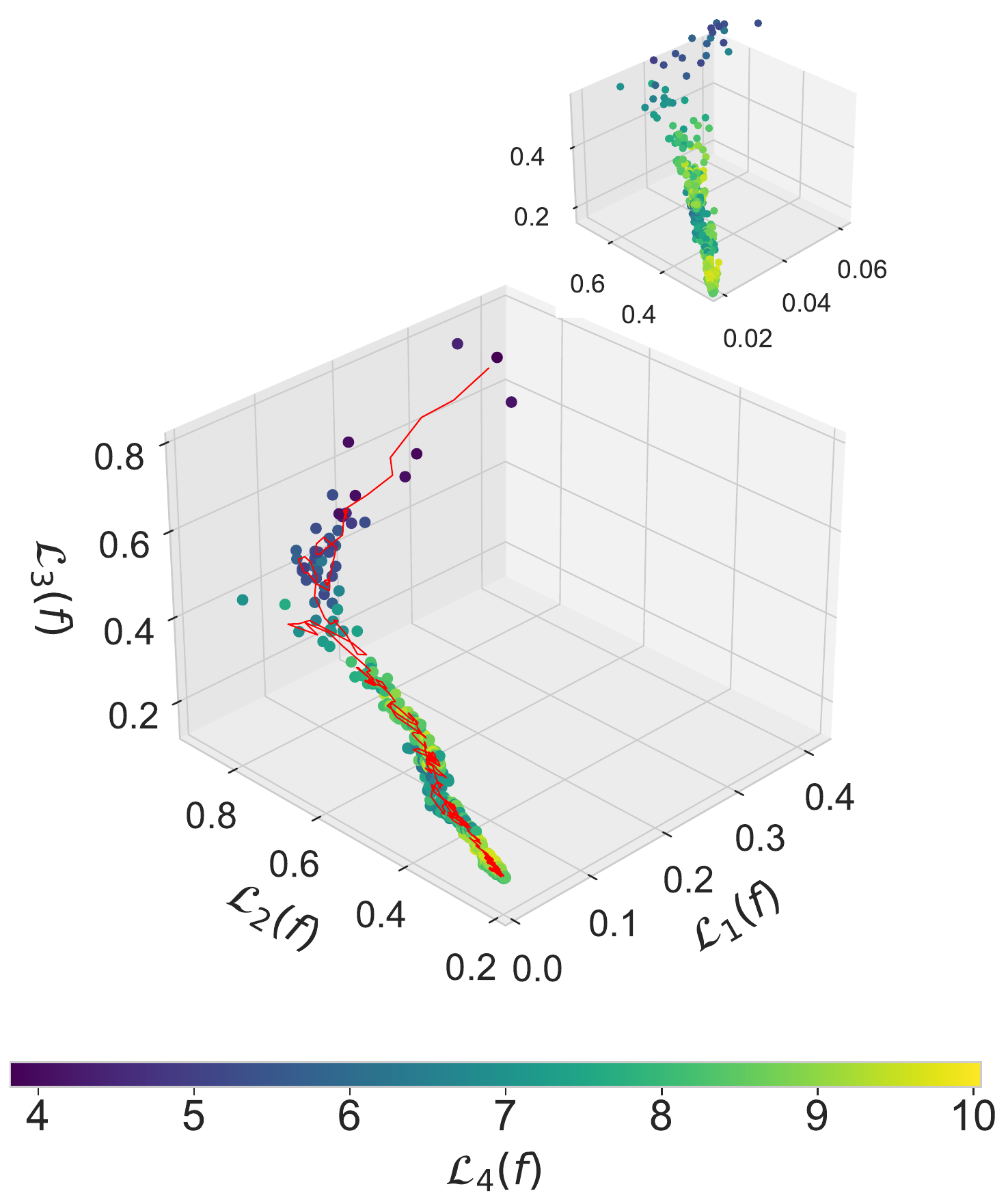}
        \centerline{(d) genbase}
    \end{minipage}
    
    \begin{minipage}{0.48\textwidth}
        \includegraphics[width=0.85\textwidth]{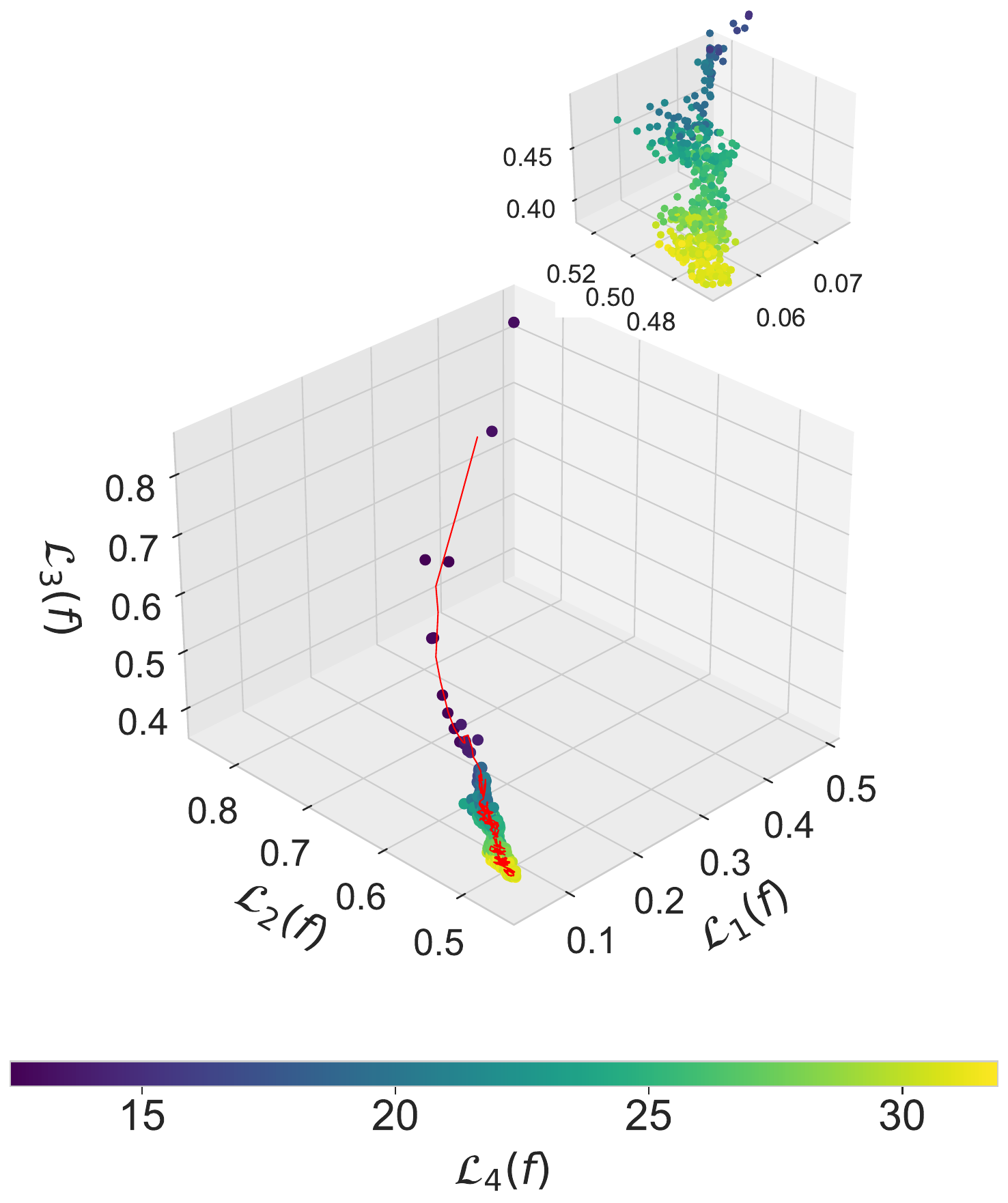}
        \centerline{(e) enron}
    \end{minipage}
    \begin{minipage}{0.48\textwidth}
        \includegraphics[width=0.85\textwidth]{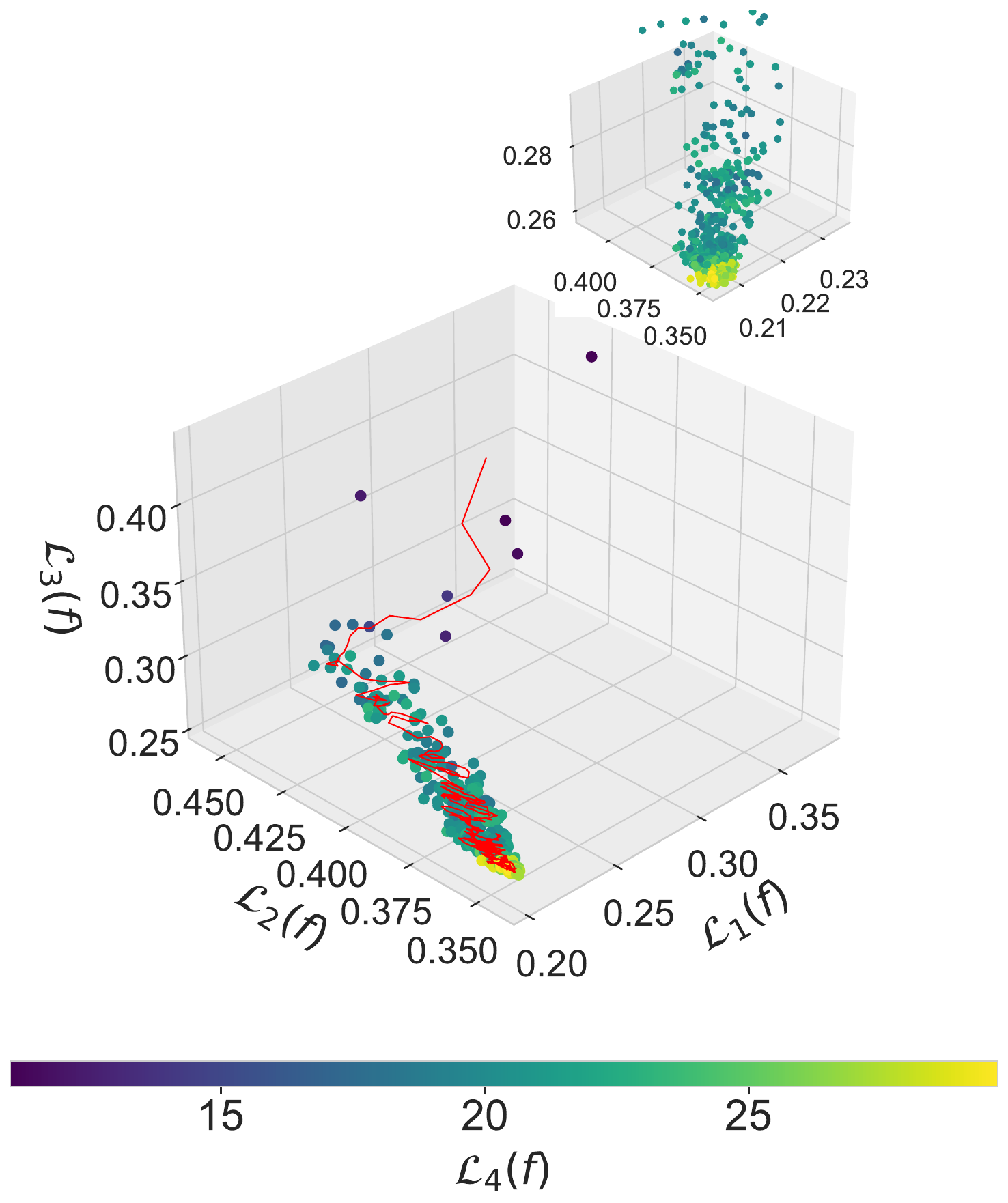}
        \centerline{(f) yeast}
    \end{minipage}

    \caption{The training curves of CLML on datasets flags through yeast (c-f).}
    \label{training_curves_supplement_1}
    \end{figure*}

    \begin{figure*}
    \centering

    \begin{minipage}{0.48\textwidth}
        \includegraphics[width=0.85\textwidth]{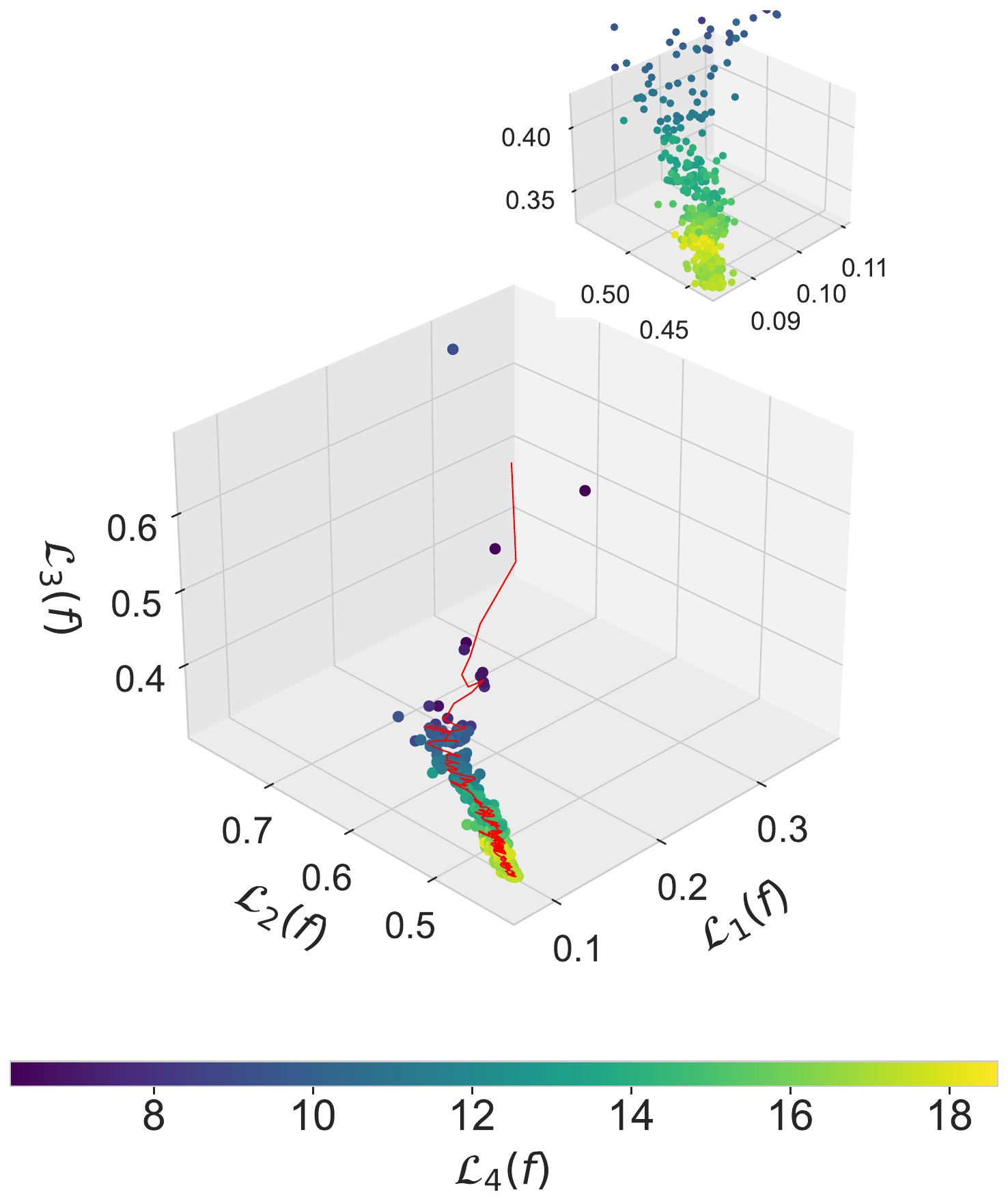}
        \centerline{(g) tmc2007-500}
    \end{minipage}
    \begin{minipage}{0.48\textwidth}
        \includegraphics[width=0.85\textwidth]{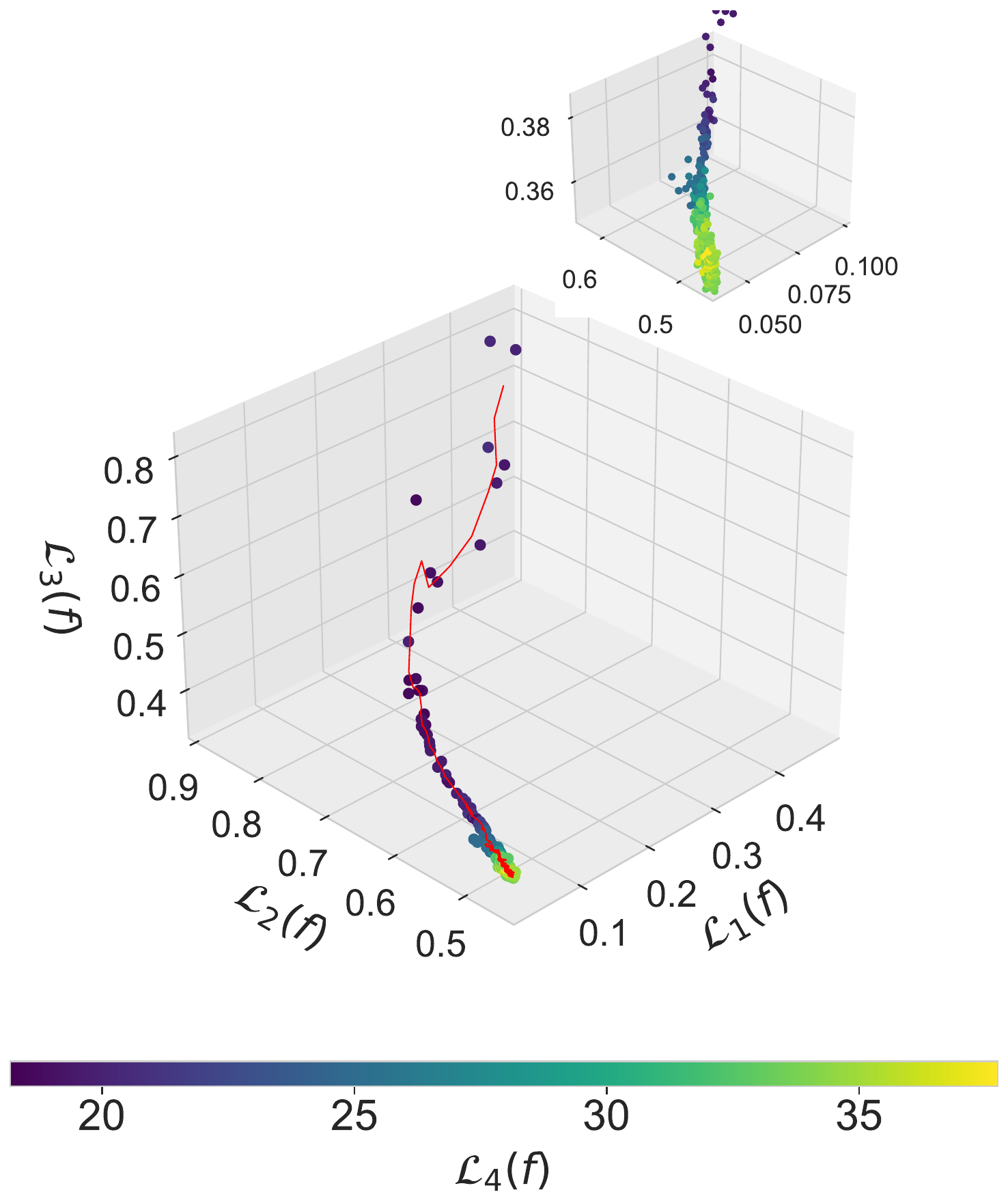}
        \centerline{(h) mediamill}
    \end{minipage}
    
    \begin{minipage}{0.48\textwidth}
        \includegraphics[width=0.85\textwidth]{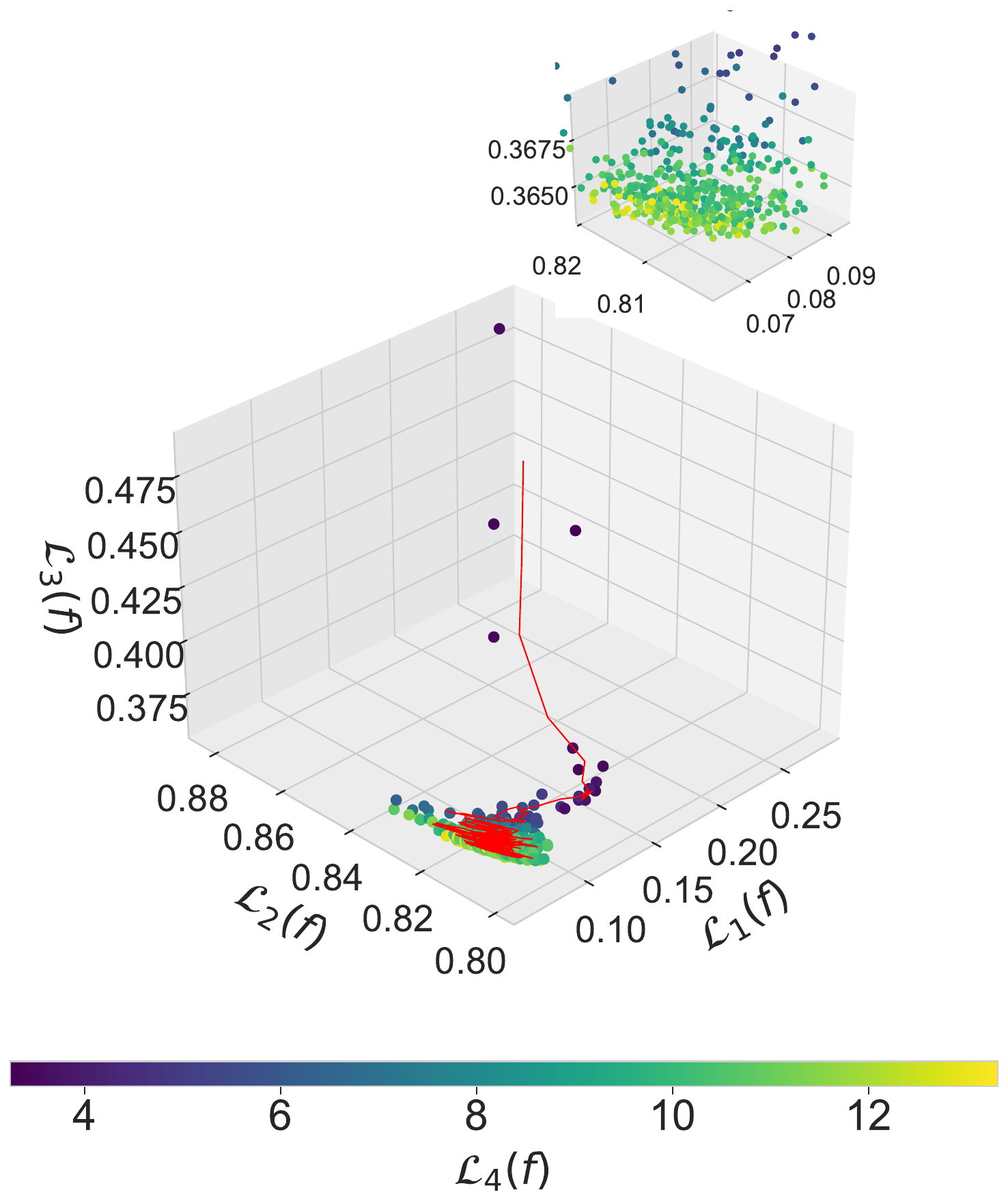}
        \centerline{(i) IMDB-F}
    \end{minipage}
    \caption{The training curves of CLML on datasets tmc2007-500 through IMDB-F (g-i).}
    \label{training_curves_supplement2}
    \end{figure*}

\section{Extended results of training curves against surrogate loss}\label{surrogate_extended}

    Figures \ref{training_curves_supplement_1} and \ref{training_curves_supplement2} plot the training curves of CLML on datasets flags through IMDB-F (c-i).

    \section{Useful definitions, corollaries, and lemmas}\label{usefuldefinitionscorollariesandlemmas}

    \begin{definition}[Metric Risk] We define the conditional and Bayes risk of $\mathcal{L}_1, \mathcal{L}_2,$ and $\mathcal{L}_3$ given $\textbf{X}$ and \textbf{Y} for $i=1,2,3$ as follows:

    \begin{equation}
    \begin{gathered}
        R_{\mathcal{L}_i}(f) = \frac{1}{N}\sum_{j=1}^{N} \sum_{\textbf{y}_j\in\mathcal{Y}}p(\textbf{y}_j|\textbf{x}_j)
        \mathcal{L}_i(f(\textbf{x}_j),\textbf{y}_j)\\
         R^B_{\mathcal{L}_i}(f) = \frac{1}{N}\sum_{j=1}^{N}
       \underset{f'}{\text{inf}}[\sum_{\textbf{y}_j\in\mathcal{Y}}p(\textbf{y}_j|\textbf{x}_j)
        \mathcal{L}_i(f'(\textbf{x}_j),\textbf{y}_j)]
    \end{gathered}
    \end{equation}
    \end{definition}
    The overall risk and Bayes risk is given by:
    \begin{equation}
    \begin{gathered}
        R_{\boldsymbol{\mathcal{L}}}(f) = (R_{\mathcal{L}_1}(f),R_{\mathcal{L}_2}(f),R_{\mathcal{L}_3}(f)) \\ R^B_{\boldsymbol{\mathcal{L}}}(f) = (R^B_{\mathcal{L}_1}(f),R^B_{\mathcal{L}_2}(f),R^B_{\mathcal{L}_3}(f))
    \end{gathered}
    \end{equation}

    \begin{corollary}[Below-bounded and Interval] The Lebesgue measure is naturally below-bounded and interval, \textit{i.e.}, for any $F, F'$ and $R, R' \subset Z$, $\lambda(H(F,R)) = \lambda(H(F',R'))$ or $|\lambda(H(F,R)) - \lambda(H(F',R'))| > 0$, which is naturally inherited from the underlying below-bounded and interval properties of $\mathcal{L}_1, \mathcal{L}_2$ and $\mathcal{L}_3$ following \citet{pmlr-v19-gao11a}.
    \end{corollary}

    \begin{lemma}[The Lebesgue Contribution Equals Lebesgue Improvement]\label{lebesgue} Let $\lambda(H(F,R))$ denote the Lebesgue measure over a set $F$. The overall improvement toward the minimisation of $\mathcal{L}_1, \mathcal{L}_2,$ and $\mathcal{L}_3$, is prescribed by the volume of $\lambda(H(F,R))$, which can be expressed as the sum of contributions of losses for each function representation $f\in F$:

    \begin{equation}
    \begin{gathered}
        \lambda(H(F,R)) = \sum_{f\in F}\lambda(P(f)) = \\\sum_{f\in F}\int_{\mathbb{R}^o}\boldsymbol{1}_{H(\{f\},R)\backslash H(F\backslash\{f\},R)}(\textbf{z})d\textbf{z}
    \end{gathered}
    \end{equation}
        
    \end{lemma}

    \begin{proof}
        Consider a redefined Lebesgue measure as the union of non-overlapping (disjoint) contribution regions for each $f\in F$. By substitution:
        \begin{equation}
        \begin{gathered}
            \lambda(H(F,R)) = \int_{\mathbb{R}^o}\textbf{1}_{H(F,R)}(\textbf{z})d\textbf{z} = \\\int_{\mathbb{R}^o}\textbf{1}_{\cup_{f\in F}H(\{f\},R)\backslash H(F\backslash\{f\},R)}(\textbf{z})d\textbf{z}
        \end{gathered}
        \end{equation}
        The integral can be re-written to express the sum over disjoint contribution regions:
        \begin{equation}
            \begin{gathered}\int_{\mathbb{R}^o}\textbf{1}_{\cup_{f\in F}H(\{f\},R)\backslash H(F\backslash\{f\},R)}(\textbf{z})d\textbf{z} = \\\sum_{f\in F}\int_{\mathbb{R}^o}\textbf{1}_{H(\{f\},R)\backslash H(F\backslash\{f\},R)}(\textbf{z})d\textbf{z} = \sum_{f\in F}\lambda(P(f)).
            \end{gathered}
        \end{equation}
    \end{proof}

\end{document}